\documentclass[preprint,12pt]{elsarticle}

\usepackage{amssymb}
\usepackage{amsmath}
\usepackage{amsthm}

\usepackage{tabularray}
\usepackage{graphicx}
\graphicspath{{./}}
\usepackage[linesnumbered]{algorithm2e}
\SetKwComment{Comment}{/* }{ */}
\RestyleAlgo{ruled}
\usepackage{multirow}
\usepackage{booktabs}
\usepackage{array}
\usepackage{caption}
\usepackage{subcaption}
\usepackage{textcomp}
\usepackage{stfloats}
\usepackage{url}
\usepackage{verbatim}
\usepackage{soul}
\usepackage[most]{tcolorbox}
\usepackage{csquotes}
\usepackage{enumitem}

\usepackage{cleveref}
\usepackage{thmtools}
\usepackage{nicefrac}
\usepackage{multirow}
\usepackage{thmtools}
\usepackage{bm}
\usepackage{bbm}
\declaretheorem[name=Theorem,numberwithin=section]{theorem}
\declaretheorem[name=Lemma,numberwithin=section]{lemma}

\declaretheorem[name=Corollary,numberwithin=section]{corollary}
\declaretheorem[name=Assumption]{assumption}

\declaretheorem[name=Remark,numberwithin=section]{remark}

\journal{Expert Systems With Applications}

\begin{document}

\begin{frontmatter}

\title{FedLBW: A Loss-Based Weighting Strategy for Federated Learning on Non-IID Data in Wireless Networks}
\author[1]{Majid Kundroo} 
\ead{kundroomajid@cbnu.ac.kr}
\author[2]{Tinku Singh} 
\ead{tinku.singh@bennett.edu.in}
\author[1]{Taehong Kim\corref{cor1}} 
\ead{taehongkim@cbnu.ac.kr}
\cortext[cor1]{Corresponding author}

\address[1]{School of Information and Communication Engineering, Chungbuk National University, Cheongju, 28644, Republic of Korea}
\address[2]{School of Computer Science Engineering and Technology,
Bennett University, Greater Noida, 201310, India}

\begin{abstract}
Federated Learning (FL) enables collaborative machine learning (ML) across distributed clients while preserving privacy. However, efficient model convergence in FL remains challenging, especially in wireless networks where non-independent and identically distributed (non-IID) data and frequent client dropouts are common. Traditional FL algorithms, such as FedAvg, rely solely on dataset size to weight client updates. This introduces biases towards clients with larger datasets and makes the process sensitive to non-IID data, outliers, and client dropouts. To address these challenges, we propose \textbf{Fed}erated Learning with \textbf{L}oss-\textbf{B}ased \textbf{W}eighting (\textit{FedLBW}), a novel aggregation method that assigns each client’s update a weight proportional to the inverse of its validation loss, computed using a small proxy dataset on the server, rather than its dataset size. This ensures that lower‑loss models exert greater influence during aggregation, prioritizing the most reliable updates and boosting overall performance.
Through extensive experiments across multiple datasets, including FashionMNIST (CNN), CIFAR-10 (ResNet-18), and CIFAR-100 (ResNet-34), we demonstrate that \textit{FedLBW} achieves higher accuracy and faster convergence compared to baseline algorithms such as FedAvg, FedAvgM, FedProx, FedNova, FedLAW and FedDkw, with notable improvements of up to 7.6\% higher accuracy on CIFAR-10 in extreme non-IID cases. Moreover, \textit{FedLBW} showcases exceptional resilience to increasing dropout probabilities, consistently maintaining significantly higher accuracy even in challenging conditions. These results establish \textit{FedLBW} as an effective and resilient solution for FL in wireless network environments, offering marked improvements in model accuracy, convergence speed, and robustness to non-IID data and client dropouts.

\end{abstract}


\end{frontmatter}

\section{Introduction}
\label{sec:introduction}
Federated Learning (FL) \cite{Konecny2016} has emerged as a promising paradigm for training Machine Learning (ML) models on decentralized data across multiple clients while preserving data privacy. In FL, clients collaborate to build a global model by locally training ML models on their respective private datasets and periodically sending model updates to a central server. The server then aggregates these client updates to form the global model. A critical aspect of this aggregation step is the weighting approach used to combine the client updates, as it can significantly impact the global model's performance and convergence \cite{WANG2025126354}, particularly in wireless network environments where the data is usually non-IID \cite{10589673,fldqn}.
\par
Classical FL algorithms, such as FedAvg \cite{BrendanMcMahan2017}, have weighted client updates based on the size of each client's dataset.
However, this approach has significant limitations.
First,in scenarios involving non-IID data, it may not prioritize updates from clients with high-quality local models \cite{hu2023element}.
This bias can lead to suboptimal global model performance, as the quantity of data does not always correlate with its quality or relevance \cite{10495552}. 
Second, this method fails to account for the non-IID nature of data, a common characteristic in FL scenarios, especially in wireless networks where data heterogeneity is prevalent. Third, the data size-based weighting strategy is susceptible to outliers, allowing poorly trained local models from clients with large datasets to influence the global model disproportionately \cite{hu2023element}. This sensitivity can potentially degrade the overall model performance. 
Fourth, client dropouts, an inevitable characteristic of wireless networks, present a significant challenge in FL. When clients disconnect during training, it can lead to oscillations around stationary points of the global loss function, potentially causing severe performance degradation, which not only hinders convergence but also impacts the overall system reliability \cite{Sun_2024}.
Lastly, this approach provides no incentive for clients to improve their local training processes, as their contributions are solely determined by the size of their datasets rather than the quality of their trained models \cite{GHASEMI2025127273}. 
These limitations collectively highlight the need for a more sophisticated weighting strategy in FL algorithms, particularly for applications in wireless network environments.
\par

To address the previously discussed limitations, we propose a novel weighting strategy \textbf{Fed}erated Learning with \textbf{L}oss-\textbf{B}ased \textbf{W}eighting (\textit{FedLBW}), which uses the inverse of the local model's validation loss values as weights during the aggregation step. Unlike the traditional approach, our performance-based weighting scheme assigns higher weights to clients with lower validation loss values, effectively giving more importance to better-performing local models. Intuitively, clients with lower validation loss values are likely to have achieved higher accuracy on their local data, implying that their model updates could have better generalization. By inversely weighting each client's validation loss, \textit{FedLBW} theoretically prioritizes more accurate local updates, thus potentially accelerating convergence and improving global model performance. 
To compute these validation loss values, \textit{FedLBW} employs a small proxy dataset on the server. The use of such proxy datasets is increasingly supported in recent FL literature for tasks such as learning optimal aggregation weights \cite{li2023revisiting}, model evaluation \cite{yang2024understanding}, hyperparameter tuning \cite{fedcust2025} and quality assessment \cite{yang2024understanding}. These proxy datasets are often available in practical FL systems and can be obtained from public repositories, making this approach both realistic and scalable. This loss-based strategy offers several advantages. First, it handles non-IID data more effectively by emphasizing well-trained local models, irrespective of their data size. Second, it mitigates the influence of outliers by reducing the weights of poorly trained local models. Crucially, it incentivizes clients to perform better local training, as their contributions to the global model are directly tied to their local model performance. This approach enables \textit{FedLBW} to adapt dynamically to client performance, focusing on updates that demonstrate higher local model performance. Consequently, loss-based weighting not only aligns with the goal of enhancing model quality but also allows the algorithm to leverage the best-performing updates in each round. This novel weighting strategy, \textit{FedLBW}, has the potential to significantly improve the performance and convergence of FL systems in wireless networks.
\par

We implemented \textit{FedLBW} algorithm and evaluate it on image classification tasks using different datasets (FashionMNIST, CIFAR-10, CIFAR-100) and model architectures (CNN, ResNet-18, ResNet-34).
To simulate realistic FL scenarios, different non-IID data distributions are considered across clients, generated through widely used Dirichlet sampling \cite{Hsu2019,li2023revisiting, access_24} controlled by a parameter (Dir $\alpha$). Extensive experimental results demonstrate the effectiveness of \textit{FedLBW} in improving the global model's accuracy and convergence, particularly in extreme non-IID settings and high client dropout probabilities. A detailed analysis is also given, comparing the proposed approach's performance with the traditional weighting method based on the number of samples and investigating the impact of varying degrees of non-IID data distribution.
\par

The main contributions of this study can be summarized as follows:
\begin{itemize}
\item We propose a novel weighting strategy (\textit{FedLBW}) for the aggregation step in FL that uses the inverse of local model's validation loss values as weights instead of number of data samples.
\item We provide a theoretical convergence analysis for \textit{FedLBW}.
\item We provide extensive empirical evidence demonstrating \textit{FedLBW}'s robustness across diverse datasets (FashionMNIST, CIFAR-10, CIFAR-100) and model architectures (CNN, ResNet-18, ResNet-34), achieving consistent performance improvements and faster convergence, particularly in extreme non-IID scenarios (Dirichlet parameter $\alpha$ = 0.1) and under low client participation ratios. 
\item We highlight \textit{FedLBW}'s exceptional resilience to client dropouts, especially in wireless networks, where it maintains stable performance even under extreme dropout conditions, outperforming both conventional methods, such as FedAvg, and SOTA methods, like FedLAW.
\end{itemize}

\par
The remainder of this paper is organized as follows: Section \ref{sec:related_works} presents a comprehensive review of relevant FL algorithms and techniques. 
Section \ref{sec:methodology} introduces the theoretical foundations and motivation behind our proposed method, \textit{FedLBW}, followed by a detailed problem formulation and an in-depth explanation of our novel loss-based weighting strategy for model aggregation. 
Section \ref{sec:convergence}, provides a theoretical convergence analysis.
In Section \ref{sec:evaluations}, we describe our experimental methodology, including the system design, dataset preparation, data partitioning methods, machine learning model architectures, and hyper-parameter configurations.
Section \ref{sec:results_discussions} presents a thorough analysis of our empirical results, demonstrating \textit{FedLBW's} superior performance in terms of accuracy, convergence, and resilience to both non-IID data and client dropouts.
Finally, Section \ref{sec:conclusions} summarizes our key findings and outlines promising directions for future research in FL systems.

\section{Related Works}
\label{sec:related_works}
The foundational and widely used FL algorithm, FedAvg \cite{BrendanMcMahan2017}, weights client updates based on the number of data samples. However, this approach suffers from several limitations, such as bias towards clients with larger datasets, ineffectiveness in handling non-IID data distributions, sensitivity to outliers, and a lack of incentives for clients to perform better local training. Zhao et al. \cite{zhao2018federated} demonstrated that the accuracy can drop by up to 55\% for highly skewed non-IID data in case of wireless networks, highlighting the ineffectiveness of traditional FL algorithms in handling such scenarios. Li et al. \cite{li2020federated} discussed the bias towards clients with larger datasets, which can negatively impact the global model's performance. Additionally, Li et al. \cite{li2021survey} emphasized the lack of an incentive mechanism in FL, suggesting that incentives could motivate clients to participate more actively.
\par

Several studies have attempted to address these biases and improve the performance of FL in non-IID settings. MOCHA \cite{smith2017federated}, which presented a communication-efficient optimization technique that trains distinct but related models for each device using a multi-task learning framework to overcome communication-related challenges in FL.
Federated averaging with momentum (FedAvgM) \cite{Hsu2019}, an extension of FedAvg, integrates adaptive momentum into the aggregation process, enhancing model convergence and robustness. By dynamically adjusting momentum based on local model updates, FedAvgM mitigates the effects of non-IID data distributions. However, it has high computational complexity and is sensitive to hyper-parameters, requiring careful tuning. While FedAvgM enhances convergence and robustness, it may struggle with extreme non-IID data distributions or adversarial client behaviors.
FedProx \cite{li2020federated} introduces a proximal term to handle heterogeneity among clients, stabilizing training by limiting the updates' deviation from the global model. 
Many other approaches are also employed on the server side, like Reddi \textit{et al.} \cite{Reddi2020}, which used adaptive optimization methods like ADAGRAD \cite{duchi2012randomized}, YOGI \cite{NEURIPS2018_90365351}, or ADAM \cite{kingma2014adam} to improve convergence, particularly with heterogeneous data.
SCAFFOLD \cite{Karimireddy2020} corrects client drift using control variates, ensuring that updates are more aligned. Despite these improvements, these methods do not fully utilize the performance potential of clients' local models and may fall short in highly non-IID environments like wireless networks.
FedAdap \cite{kundroo2023fedada} addresses client heterogeneity through adaptive client-side hyper-parameter optimization, dynamically adjusting learning rates and epochs based on local training metrics to reduce convergence time by up to 82.5\% in non-IID settings.
\par

Advanced aggregation techniques have been developed to further enhance FL performance. FedNova \cite{wang2021novel}, enhances the aggregation process in FL by implementing a normalization technique. This approach addresses the challenges posed by disparities in client update frequencies and learning rates. While FedNova's normalization method proves beneficial in balancing contributions across non-IID data, it may inadvertently diminish the influence of clients possessing larger or more informative datasets. Consequently, this could potentially constrain the overall performance of the global model.
FedMA \cite{DBLP:conf/iclr/WangYSPK20} matches and averages neural network layers to improve the aggregation process.
Pillutla et al. \cite{pillutla2022robust} introduced a method leveraging the geometric median for aggregating updates, enhancing resilience against potential poisoning of local data or model parameters. 
Jiale et al \cite{ZHANG2025126418} proposed SFFL which decomposes the clients’ training objectives into fair training objects based on underlying distributions to improve local fairness and model performance. 
CESA \cite{WANG2024103997} used the concept of sparse graphs and minimal spanning trees to reduce the running time by 28.2\% and also reduces the communication cost significantly.
Wang et al. \cite{9500928} assigns higher or lower weights based on user reputation scores, enhancing adaptability compared to the standard FedAvg.
Li et al. introduced FedLAW \cite{li2023revisiting}, which utilizes a server-side proxy dataset to optimize client contribution weights during aggregation. While this approach enhances model performance, it incurs additional computational costs and requires a proxy dataset.
In order to improve the performance of FL systems in wireless networks, Geng et al. \cite{geng2024adaptive} proposed an independent client sampling strategy that considers data and system heterogeneity, enhancing training time and performance.
Chen et al. \cite{10589673}  proposed a gradient calibration mechanism which calibrates the global update by adopting clients’ historical gradients to avoid biased global model updating in wireless networks.

\par
Although innovative, these advanced techniques do not fully resolve the issues of weight imbalance caused by varying data distributions in non-IID settings. Specifically, they fall short in scenarios where client data distributions are highly heterogeneous, resulting in a global model that is not representative of the overall data distribution. The inability to effectively balance the contributions of clients, particularly when their data sizes and distributions vary significantly, remains a challenge.

\par
To address this gap, we propose \textit{FedLBW}. Unlike previous methods, \textit{FedLBW} leverages the inverse of local model's validation loss values as weights during aggregation. This approach ensures that clients with better-performing models have a more significant impact on the global model, thereby mitigating the adverse effects of non-IID data. By using the inverse of local model's validation loss values as weights, our approach effectively deals with non-IID data and reduces the impact of outliers. Additionally, it incentivizes clients to improve their local training performance, as their contributions to the global model directly correlate with their local model performance. This unified approach enhances the overall robustness and effectiveness of FL systems, ensuring more accurate and reliable model updates across diverse and dynamic datasets.

\section{FedLBW: Federated Learning with Loss-Based Weighting}
\label{sec:methodology}
\subsection{Motivation and Overview}

This study presents a novel weighting strategy for the aggregation step in FL. Unlike the traditional approach that weights client updates based on the number of data samples each client possesses, our performance-based weighting scheme assigns higher weights to clients with better-performing local models.
Figure \ref{fig:fedlbw_overview} shows a general overview of the proposed approach where a FL setup, involving multiple clients communicating with a central server over a wireless network. Each client has its own local dataset drawn from a potentially different data distribution (non-IID data). The goal is to learn a global model that performs well across all clients data distributions. 

\begin{figure}[!t] 
\centering
\includegraphics[width=0.85\linewidth,keepaspectratio]{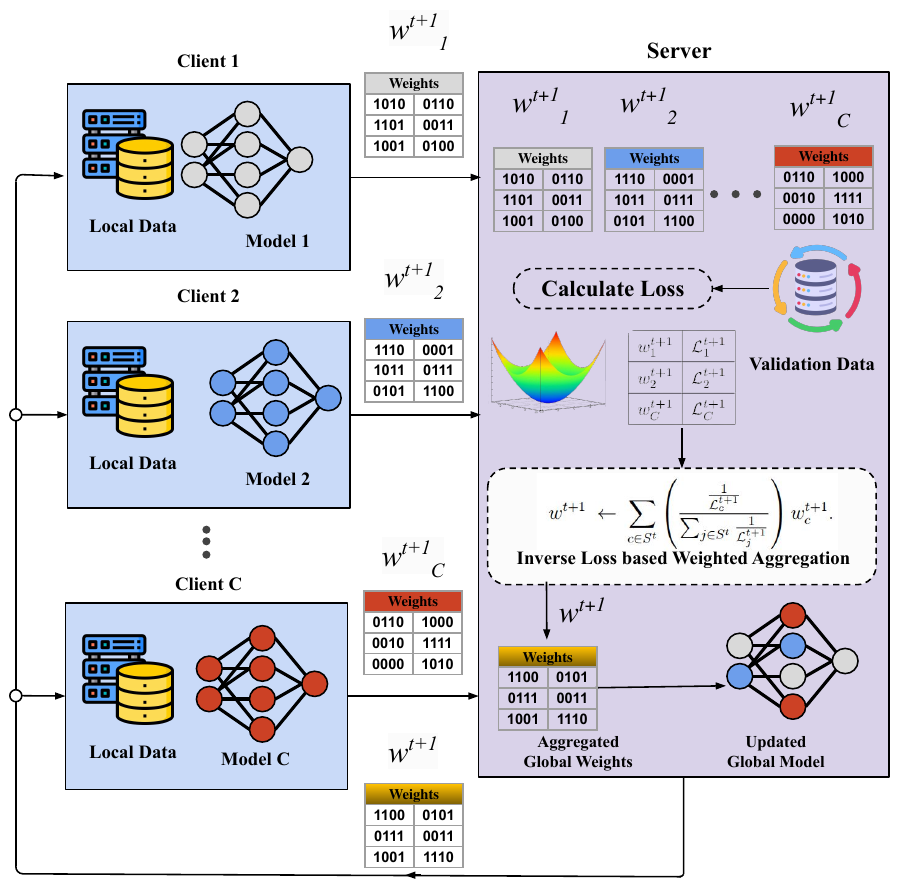}
\caption{Overview of the proposed Federated Learning with Loss-Based Weighting (\textit{FedLBW}) framework consisting of a server and $C$ clients. The server distributes the global model to participating clients, which train on their local data. During aggregation, client updates are weighted based on their inverse validation loss values computed on a small proxy dataset, prioritizing better-performing models. The weighted updates are then aggregated to form the new global model for the next round of training.}
\label{fig:fedlbw_overview}
\end{figure}

\par
\textit{FedLBW} replaces the traditional weighting approach based on the number of samples with the inverse of the local model's validation loss value. The validation loss value measures how well a model performs on a given task, such as classification. Lower loss values indicate better-performing models. During the aggregation step, instead of weighting by the number of samples, \textit{FedLBW} weight each client's update by the inverse of its local model's validation loss value, which the server computes using a small proxy dataset. This means that clients with lower validation loss values (better-performing local models) will have higher weights, and their updates will have a greater influence on the global model. Consequently, well-trained local models are prioritized regardless of their data size, effectively addressing non-IID data distributions commonly encountered in wireless networks and reducing the impact of poorly trained outlier models.

\subsection{Problem Formulation}
Consider a FL system with a set of $C$ clients indexed by $c \in \{1,2,\cdots,C\}$ connected to a central server over a wireless network. Each client $c$ has a local dataset $P_c = \{x_c^k, y_c^k\}$, where $x_c^k$ represents the input features and $y_c^k$ represents the corresponding labels of the $k$-th data sample. $n_c$ is the number of samples at client $c$, $N = \sum_{c}^{C} n_c$ are the total number of samples across all clients, and $m$ is the number of classes.

For each client $c$, the local optimization problem can be defined as:

\begin{equation}
\label{eq:local_optimization}
\min_{w} F_c(w) \;=\; \frac{1}{n_c} \sum_{k=1}^{n_c} 
\left( - \sum_{j=1}^{m} y_c^{k(j)} \,\log p_w(x_c^k)^{(j)} \right),
\end{equation}
where:
\begin{itemize}
    \item $w$ represents the global model parameters.
    \item \(p_w(x) = \mathrm{softmax}(f_w(x)) \in \Delta^{m}\) is the predicted class-probability vector.
    \item \(\Delta^{m} = \{\, p \in [0,1]^m \mid \sum_{j=1}^{m} p^{(j)} = 1 \,\}\)
    \item \(y_c^{k(j)} \in \{0,1\}\) is the one-hot label for class \(j\) of sample \(k\) at client \(c\). 
    \item \(m\) is the number of classes, and \(n_c\) is the number of samples at client \(c\).
    \item $F_c(w)$ is the local objective function at client $c$ (Equation \ref{eq:local_optimization}).
\end{itemize}
\par
The global optimization problem of the overall FL system can be formulated as :
\begin{equation}
\centering
\label{eq:global-optimization}
\min_{w} F(w) = \sum_{c=1}^C \frac{n_c}{N} F_c(w)
\end{equation}

\par
In traditional FedAvg (assuming client participation constraint: $|S^t| = \max(| C | \cdot \text{CPR}, 1)$.), the aggregation of weights are determined by:

\begin{equation}
\centering
\label{eq:fedavg_aggregation}
w^{t+1} \;\gets\; \sum\nolimits_{c\in S^t} \frac{n_c}{\sum_{j\in S^t} n_j}\,w_{c}^{t+1}.
\end{equation}

\par 
Our proposed \textit{FedLBW} modifies these weights based on the inverse of local model's validation loss values as :

\begin{equation}
\label{eq:fedlbw_weights}
w^{t+1} \;\gets\; \sum_{c\in S^t} \left( \frac{\frac{1}{\mathcal{L}^{t+1}_c}}{\sum_{j\in S^t} \frac{1}{\mathcal{L}^{t+1}_j}} \right) w^{t+1}_c .
\end{equation}

where $\mathcal{L}_c^{t+1}$ is the validation loss of client $c$'s returned model $w_c^{t+1}$ on the server-side proxy dataset $V$, computed as $\mathcal{L}_c^{t+1} = \frac{1}{|V|}\sum_{(x,y)\in V}\ell\!\left(f_{w_c^{t+1}}(x),\,y\right)$.
\par
The final objective function of \textit{FedLBW} can defined as:

\begin{equation}
    \label{eq:final_objective}
    \centering
    \min_{w} F(w) = \sum_{c \in S^t} \left( \frac{\frac{1}{\mathcal{L}^{t+1}_c}}{\sum_{j \in S^t} \frac{1}{\mathcal{L}^{t+1}_j}} \right) F_c(w)
\end{equation}

subject to:
\begin{itemize}
    \item Local model updates: $w^{t+1}_c = w_t - \eta \nabla F_c(w^t)$.
    \item Client participation constraint: $|S^t| = \max(| C | \cdot \text{CPR}, 1)$.
\end{itemize}

where:
\begin{itemize}
    \item $t$ denotes the communication round.
    \item $\eta$ is the learning rate.
    \item $S^t$ is the set of selected clients in round $t$.
    \item CPR is the Client Participation Ratio.
\end{itemize}

\par
The key difference in our formulation lies in the weighting strategy, where we replace the traditional data size-based weights with loss-based weights, allowing the algorithm to prioritize better-performing local models during aggregation.

\subsection{\textit{FedLBW} Algorithm}
The \textit{FedLBW} is implemented within the FL framework, where the server coordinates the training process and aggregates the updates from multiple clients, as outlined in Figure \ref{fig:fedlbw_overview} and Algorithm \ref{alg:proposed}.

\begin{algorithm}[!t]
\caption{FedLBW}
\small
\label{alg:proposed}
Initialize $w_{0}$

\For{each round $t = 1, 2, \dots $}{
$a \gets \max\{1, \lfloor |C| \cdot \mathrm{CPR} \rfloor\}$\;
 \Comment{$|C|$ is the total number of clients and CPR is client participation ratio}
$ S^{t} \gets (random \ set \ of \ a $ clients)\;
\For{each client $ c \in \mathcal S^{t} $ \textbf{in parallel} }{
    $ {w^{t+1}_c}  \gets $ \text{ClientUpdate($ c,w^t $)}\;
    $\mathcal{L}^{t+1}_c \gets \text{CalculateLoss($w^{t+1}_c$)}$\;
    }
    $ w^{t+1} \gets \sum_{c \in S^t} \left( \frac{\frac{1}{\mathcal{L}^{t+1}_c}}{\sum_{j \in S^t} \frac{1}{\mathcal{L}^{t+1}_j}} \right) w^{t+1}_c$\;
    \Comment{Weighted aggregation based on the inverse validation loss of each client}
}

\SetKwFunction{FMain}{ClientUpdate}
    \SetKwProg{Fn}{Function}{:}{End Function}
    \Fn{\FMain{$c, w$}}
    {
    $\mathcal{B} \gets (split\ P_c\ into\ batches\ of\ size\  B) $\;
    \Comment{${P_c}$ is\ client c\textquotesingle s\ local\ data}

    \For{each epoch $e\ from\ 1\ to\ E $} {
    \For{batch $b \in \mathcal{B} $ }{
    $w\gets w - \eta \nabla l(w;b)$\;
        }
    }
    \KwRet{$w$ to server}\;
    }

\SetKwFunction{FMain}{CalculateLoss}
    \SetKwProg{Fn}{Function}{:}{End Function}
    \Fn{\FMain{$w$}} 
    {
    \Comment{calculates the validation loss of the model on a small proxy dataset $V$}
    $val\_loss \gets \frac{1}{|V|}\sum_{(x,y)\in V} \ell\!\left(f_w(x),\, y\right)$\;
    \KwRet{$val\_loss$}\;
    \Comment{returns the validation loss}
    }
\end{algorithm}

At the beginning of each training round ($t$), the number of participating clients ($a$) is determined by the total number of clients ($|C|$) and a predefined CPR (line 3). This step ensures that not all clients participate in every round, which is a common practice in FL to reduce communication costs and improve scalability. Then server selects a random set of clients ($S^t$) to participate in the current round (line 4). 

Each participating client ($c \in S^t$) executes $ClientUpdate$ function (lines 11-19) in parallel. This function performs local training on the client's data ($P_c$) for a fixed number of epochs $E$ ($max\_epochs$) using stochastic gradient descent (SGD) or a variant thereof. During local training, the client's local model ($w^t_c$) is updated based on the gradients computed from the local data batches ($b$) and the learning rate ($\eta$) and each client returns the updated weights to the server for aggregation.
\par
On the server, the updated local models ($w^{t+1}_c$) are gathered.
The validation loss values ($\mathcal{L}^{t+1}_c$) for each of these models are then computed on server side with a small class-balanced proxy dataset (line 7) using a cross entropy loss function (lines 20-23).

In the aggregation step line 10, the server updates the global model ($w^{t+1}$) by taking a weighted average of all the client updates. However, instead of weighting each client's update by the number of data samples in the traditional federated algorithms like FedAvg, \textit{FedLBW} uses the inverse of the validation loss values as weights.
This ensures that clients with lower validation loss values (better-performing local models) contribute more to the global model, while clients with higher validation loss values (poorer-performing local models) have a reduced influence.

\subsection{Server-Side Proxy Dataset in FedLBW}
While traditional FL operates without server-side data, recent methodological advancements have shown that strategic use of minimal proxy data can enhance performance without compromising privacy principles. \textit{FedLBW} employs a small, class-balanced proxy dataset on the server to quantitatively assess client model contributions. This assessment mechanism enables the computation of validation loss values that subsequently determine the aggregation weights in our loss-based weighting strategy, allowing the server to prioritize higher-quality client updates.

This approach aligns with established research in FL. Li et al. \cite{li2023revisiting} demonstrated in FedLAW that a minimal proxy dataset (10 samples per class) effectively determines client contribution importance. Similarly, FedCust \cite{fedcust2025} implements proxy-based evaluation for aggregation quality assessment.

The scientific validity of this methodology is further supported by broader FL research. SA-FL \cite{yang2024understanding} confirms such proxy datasets are \enquote{already available in many FL systems}, while Scott and Cahill \cite{scott2024improved} document that servers commonly maintain \enquote{related proxy data} from public sources or consenting participants. This approach maintains FL's privacy guarantees, as client training data remains distributed while the proxy dataset serves exclusively for evaluation—representing a balanced approach between performance optimization and privacy preservation.

\subsection{Advantages and Limitations}
By assigning higher weights to better-performing local models, our approach offers several advantages over the traditional weighting method:

\begin{enumerate}
\item \textbf{Handling Non-IID Data :} In FL settings, clients often have non-IID data, meaning that their local data distributions can vary significantly. Our proposed weighting strategy can handle non-IID data more effectively by emphasizing well-trained local models, irrespective of their data size.
\item \textbf{Mitigating Outlier Influence :} Our approach mitigates the influence of outliers, which are clients with poorly trained local models. By reducing the weights of these clients, their negative impact on the global model is minimized.
\item \textbf{Improved Global Model Performance : } By emphasizing well-performing local models and mitigating the influence of outliers, our proposed weighting strategy can potentially lead to faster convergence and better global model performance, particularly in non-IID settings.
\end{enumerate}
\par
Despite its potential benefits, our approach has limitations and presents opportunities for future research. One such limitation of \textit{FedLBW} lies in its underlying assumption regarding the availability and computability of local model's validation loss values during federated training. This assumption is reasonable for many ML tasks, particularly classification and regression, but it may not hold true for certain models or applications where validation loss values are difficult to access or interpret.
Furthermore, \textit{FedLBW} requires the server to calculate model's validation loss for each client, introducing additional computational overhead compared to FedAvg. This overhead can be expressed as $O(|V||S^t|)$, where $|V|$ represents the validation set size and $|S^t|$ denotes the number of participating clients in round $t$.
This computational overhead is generally negligible in practical scenarios, since servers usually have more computational resources than individual clients and the validation set is typically much smaller than the training set.

\section{Convergence Analysis of FedLBW}
\label{sec:convergence}
We analyze Algorithm~\ref{alg:proposed} in the realistic regime of partial client participation and $E\ge 1$ local epochs per round, with model averaging using loss-based weights computed on a server-side proxy set. The proof follows key assumptions as followed by many FL papers \cite{Karimireddy2020,Li2019} as well as introduces explicit terms for (i) client sampling, (ii) local drift (due to $E>1$), and (iii) the effects from loss-based aggregation.

\subsection{Notations for Analysis}

Throughout this section we use the following notation:
\begin{itemize}
\setlength{\itemsep}{1pt}
  \setlength{\parskip}{1pt}
  \item $L_f$: the smoothness constant of the local objectives.
  \item $\sigma^2$: the variance bound of the stochastic gradients.
  \item $B$: the mini-batch size (as in Algorithm~\ref{alg:proposed}).
  \item $G$: a uniform bound on the gradient norm.
  \item $\zeta^2$: the client heterogeneity parameter; it measures the variance of local gradients from the global gradient.
  \item $g_c^{t,s}$: the stochastic gradient computed by client $c$ in round $t$ at local step $s$.
  \item $\displaystyle \beta_c^t = \frac{1/\mathcal{L}_c^{t+1}}{\sum_{j\in S^t} 1/\mathcal{L}_j^{t+1}}$: the loss-based aggregation weight for client $c$ in round $t$ and $\beta^t = \sum_{c\in S^t}\beta_{c}^t$.
  \item $\rho$: the weight concentration factor.
  \item $B_{\max}$: the maximum deviation of aggregation weights from uniform.
  \item $c \in C$ denotes the client index.
  \item $w \in \mathbb{R}^d$ denotes the model parameter vector.
  \item $w^t$ denotes the global model parameters for round $t$, $w_{c}^t$ is the model parameters of client $c$ in round $t$ and $w^{t,s}_c$ denotes the model parameters of client $c$ in round $t$ and step $s$.
  \item $F^* = \min_w F(w)$: the optimal value of the global objective.
\end{itemize}
All norms $\|\cdot\|$ refer to the $\ell_2$ norm unless stated otherwise.

\subsection{Assumptions}

\begin{assumption}[Smoothness]
\label{ass:smooth}
Each $F_c$ is $L_f$-smooth: For all $w,v \in \mathbb{R}^d$,
\[
\|\nabla F_c(w)-\nabla F_c(v)\|\le L_f\|w-v\|.
\]
\end{assumption}

\begin{assumption}[Stochastic gradients]
\label{ass:stoch}
For all clients $c$, rounds $t$ and local step ($s$), given the local model $w_c^{t,s}$, the mini-batch gradient is unbiased with bounded variance:
\[
\mathbb{E}\!\left[g_{c}^{t,s}\mid w_{c}^{t,s}\right]\!=\!\nabla F_c(w_{c}^{t,s}),\qquad 
\mathbb{E}\!\left[\|g_{c}^{t,s}-\nabla F_c(w_{c}^{t,s})\|^2\mid w_{c}^{t,s}\right]\!\le\!\frac{\sigma^2}{B},
\]
and gradients are bounded: $\|\nabla F_c(w)\|\le G$ for all $c,w$.
\end{assumption}

\begin{assumption}[Client heterogeneity]
\label{ass:hetero}
For all $w$,
\[
\frac{1}{C}\sum_{c=1}^C \|\nabla F_c(w)-\nabla F(w)\|^2\;\le\;\zeta^2.
\]
\end{assumption}

\begin{assumption}[Partial participation]
\label{ass:pp}
At each round $t$, $S^t$ is sampled uniformly at random without replacement with $|S^t|$ independently of the mini-batch sampling, and $S^t$ denotes the set of participating clients at round $t$.
\end{assumption}

\begin{assumption}[Loss-based weights]
\label{ass:weights}
For each round $t$, the server computes weights $\beta_c^t = \frac{1/\mathcal{L}_c^{t+1}}{\sum_{j \in S^t} 1/\mathcal{L}_j^{t+1}}$ where $\mathcal{L}_c^{t+1}$ is client $c$'s validation loss on proxy dataset $V$. These weights satisfy:
\begin{enumerate}[label=(\alph*)]
\item Bounded deviation: $\|\beta^t - \bar{\beta}^t\|_2 \le B_{\max}$ where $\bar{\beta}^t = 1/|S^t|$ (uniform weights)
\item Concentration bound: $\sum_{c\in S^t}(\beta_c^t)^2 \le \frac{\rho}{|S^t|}$ for some $\rho \in [1, |S^t|]$
\end{enumerate}
\end{assumption}

\subsection{Key Technical Bounds}

Let $w_{c}^{t,0}=w^t$ and $w_{c}^{t,s+1}=w_{c}^{t,s}-\eta g_{c}^{t,s}$ for $s=0,\dots,E-1$, and define $\Delta_c^t:=w_{c}^{t,E}-w^t$.

\begin{lemma}[Local drift]
\label{lem:local-drift}
By assumption \ref{ass:smooth} and \ref{ass:stoch}, for any $t,c$,
\[
\left\|\sum_{s=0}^{E-1}\bigl(\nabla F_c(w_{c}^{t,s})-\nabla F_c(w^t)\bigr)\right\|\;\le\;\frac{L_f\,\eta\,G}{2}\,E(E-1).
\]
Consequently,
\[
\left\|\sum_{c\in S^t}\beta_c^t\sum_{s=0}^{E-1}\bigl(\nabla F_c(w_{c}^{t,s})-\nabla F_c(w^t)\bigr)\right\|\;\le\;\frac{L_f\,\eta\,G}{2}\,E(E-1).
\]
\end{lemma}


\begin{lemma}[Sampling and weighting bias]
\label{lem:bias}
Let $u^t := \sum_{c\in S^t}\beta_c^t\,\nabla F_c(w^t) - \nabla F(w^t)$. Then
\[
\mathbb{E}\bigl[\|u^t\|^2\mid w^t\bigr]
\;\le\;
2\,\frac{C-|S^t|}{|S^t|(C-1)}\,\zeta^2
\;+\;
2\,G^2\,\Bigl(\sum\nolimits_{c\in S^t}(\beta_c^t)^2 - \tfrac{1}{|S^t|}\Bigr).
\]
Consequently, under Assumption~\ref{ass:weights}(b),
\[
\mathbb{E}\bigl[\|u^t\|^2\mid w^t\bigr]
\;\le\;
2\,\frac{C-|S^t|}{|S^t|(C-1)}\,\zeta^2
\;+\;
2\,G^2\,(\rho-1).
\]
\end{lemma}

\begin{lemma}[Aggregated noise]
\label{lem:noise}
Let $\xi_{c}^{t,s}:=g_{c}^{t,s}-\nabla F_c(w_{c}^{t,s})$ and $M_t:=\sum_{c\in S^t}\beta_c^t\sum_{s=0}^{E-1}\xi_{c}^{t,s}$. Then
\[
\mathbb{E}\bigl[\|M_t\|^2\mid \{w_{c}^{t,s}\}\bigr]\;\le\;\frac{\rho}{|S^t|}\,\frac{E\,\sigma^2}{B}.
\]
\end{lemma}

\subsection{Main Result}
\begin{theorem}[FedLBW with partial participation and $E>1$]
\label{thm:fedlbw}
Suppose the above assumptions hold and the stepsize satisfies $\eta E\le \frac{1}{4L_f}$. Then
\begin{equation}\label{eq:mainbound}
\scalebox{0.85}{$\displaystyle
\begin{aligned}
\frac{1}{T}\sum_{t=0}^{T-1}\mathbb{E}\big[\lVert\nabla F(w^t)\rVert^2\big]
&\le \frac{4\big(F(w^0)-F^*\big)}{\eta\,E\,T} \\[4pt]
&\quad + 8L_f\,\eta\,E\!\left(\frac{\rho}{|S^t|}\frac{\sigma^2}{B}
    +\frac{C-|S^t|}{|S^t|(C-1)}\zeta^2 + G^2(\rho-1)\right)
    + c_d\,L_f^3\,\eta^3\,G^2\,E^3.
\end{aligned}
$}
\end{equation}

where $c_d>0$ denotes an absolute constant arising from the application of Young's inequality to bound the cross terms.
\end{theorem}


\begin{proof}[Proof]
We begin with the $L_f$-smoothness of $F$:
\begin{align}
F(w^{t+1}) &\le F(w^t) + \langle \nabla F(w^t), w^{t+1} - w^t \rangle + \frac{L_f}{2}\|w^{t+1} - w^t\|^2 \label{eq:smooth}
\end{align}

From Algorithm~\ref{alg:proposed} and the update rule, we have:
\begin{align}
w^{t+1} &= \sum_{c \in S^t} \beta_c^t w_c^{t,E} = \sum_{c \in S^t} \beta_c^t \left(w^t - \eta \sum_{s=0}^{E-1} g_c^{t,s}\right) \\
w^{t+1} - w^t &= -\eta \sum_{c \in S^t} \beta_c^t \sum_{s=0}^{E-1} g_c^{t,s}
\end{align}

Decompose the gradient sum as; where $\xi_c^{t,s} = g_c^{t,s} - \nabla F_c(w_c^{t,s})$ is the stochastic noise.

\begin{equation}\label{eq:local-decomp}
\begin{split}
\sum_{s=0}^{E-1} g_c^{t,s}
&= \sum_{s=0}^{E-1} \nabla F_c\big(w_c^{t,s}\big)
  + \sum_{s=0}^{E-1}\big(g_c^{t,s}-\nabla F_c(w_c^{t,s})\big)\\
&= E\,\nabla F_c(w^t)
  + \sum_{s=0}^{E-1}\big(\nabla F_c(w_c^{t,s})-\nabla F_c(w^t)\big)
  + \sum_{s=0}^{E-1}\xi_c^{t,s},
\end{split}
\end{equation}

Define:
\begin{align}
G_t &= \sum_{c \in S^t} \beta_c^t \nabla F_c(w^t) = \nabla F(w^t) + u^t \\
D_t &= \sum_{c \in S^t} \beta_c^t \sum_{s=0}^{E-1} (\nabla F_c(w_c^{t,s}) - \nabla F_c(w^t)) \\
M_t &= \sum_{c \in S^t} \beta_c^t \sum_{s=0}^{E-1} \xi_c^{t,s}
\end{align}

Then $w^{t+1} - w^t = -\eta E G_t - \eta D_t - \eta M_t$. Substituting into \eqref{eq:smooth}:
\begin{equation}\label{eq:bound}
\begin{split}
F(w^{t+1}) &\le F(w^t) - \eta E \langle \nabla F(w^t), G_t \rangle \\
           &\quad - \eta \langle \nabla F(w^t), D_t + M_t \rangle
           + \frac{L_f \eta^2}{2}\,\big\lVert E G_t + D_t + M_t\big\rVert^2.
\end{split}
\end{equation}


Taking expectation and using $\mathbb{E}[M_t | w^t] = 0$:
\begin{equation}\label{eq:exp-bound}
\begin{split}
\mathbb{E}\big[F(w^{t+1})\big] &\le \mathbb{E}\big[F(w^t)\big]
  - \eta E \,\lVert\nabla F(w^t)\rVert^2
  - \eta E\,\mathbb{E}\big[\langle \nabla F(w^t), u^t \rangle\big] \\
&\quad - \eta\,\mathbb{E}\big[\langle \nabla F(w^t), D_t \rangle\big]
  + \frac{L_f\eta^2}{2}\,\mathbb{E}\big[\lVert E G_t + D_t + M_t\rVert^2\big].
\end{split}
\end{equation}


Using Young's inequality $\langle a, b \rangle \le \frac{1}{2\gamma}\|a\|^2 + \frac{\gamma}{2}\|b\|^2$ with appropriate $\gamma$, the inequality $(a+b+c)^2 \le 3(a^2 + b^2 + c^2)$, and Lemmas~\ref{lem:local-drift}--\ref{lem:noise}:

\begin{equation}\label{eq:variance-bound-2lines}
\begin{split}
\mathbb{E}\big[\lVert E G_t + D_t + M_t\rVert^2\big]
&\le 3E^2\,\mathbb{E}\big[\lVert G_t\rVert^2\big] + 3\,\mathbb{E}\big[\lVert D_t\rVert^2\big] \\
&+ 3\,\mathbb{E}\big[\lVert M_t\rVert^2\big] 
\le 3E^2\big(\lVert\nabla F(w^t)\rVert^2 + \mathbb{E}\big[\lVert u^t\rVert^2\big]\big) \\
  &+ 3\Big(\frac{L_f\eta G E(E-1)}{2}\Big)^{\!2} + 3\frac{\rho E \sigma^2}{\lvert S^t\rvert B}.
\end{split}
\end{equation}


Choosing $\eta E \le \frac{1}{4L_f}$ ensures the coefficient of $\|\nabla F(w^t)\|^2$ remains negative. Rearranging and telescoping from $t=0$ to $T-1$ yields the stated bound.
\end{proof}

\begin{corollary}[Nonconvex rate]
With $\eta=\Theta(1/(E\sqrt{T}))$,
\[
\scalebox{0.90}{$\displaystyle
\frac{1}{T}\sum_{t=0}^{T-1}\mathbb{E}\|\nabla F(w^t)\|^2 =
\mathcal{O}\!\left(\frac{1}{\sqrt{T}}\right)
+\mathcal{O}\!\left(\frac{1}{\sqrt{T}}\left(\frac{\rho}{|S^t|}\frac{\sigma^2}{B}
+\frac{\zeta^2}{|S^t|}+G^2(\rho-1)\right)\right)
+\mathcal{O}\!\left(\frac{1}{T^{3/2}}\right)
$}
\]
\end{corollary}

\begin{corollary}[Comparison with FedAvg]
Under FedAvg with $\beta_c^t = \frac{n_c}{\sum_{j \in S^t} n_j}$, we have
\[
\sum_{c\in S^t} (\beta_c^t)^2 \in \Big[\tfrac{1}{|S^t|},\,1\Big],
\quad\text{so}\quad \rho \;=\; |S^t| \sum_{c\in S^t} (\beta_c^t)^2 \in [1,\,|S^t|].
\]
In contrast, if FedLBW weights satisfy a bounded loss ratio $\mathcal{L}_{\min}/\mathcal{L}_{\max} \ge \delta>0$,
then $\sum_{c\in S^t} (\beta_c^t)^2 \le 1/(\delta\,|S^t|)$ and hence $\rho \le 1/\delta$.
Thus, when $\delta$ is not too small, FedLBW yields smaller $\rho$ (and a smaller bias term $G^2(\rho-1)$) than FedAvg.
\end{corollary}

\begin{remark}[Interpretation and practice]
The bound separates: (i) optimization error ($\propto 1/(\eta E T)$), (ii) stochastic noise ($\propto \rho/(|S^t|B)$), (iii) sampling–heterogeneity ($\propto \zeta^2/|S^t|$), (iv) weighting bias ($\propto B_{\max}^2$), and (v) local drift ($\propto \eta^3 E^3$). Larger $|S^t|$, larger $B$, smaller $\rho$ and $B_{\max}$, and moderate $E$ improve constants. Enforce $\rho$ and $B_{\max}$ by $\varepsilon$-smoothing, clipping, and temporal smoothing of loss-based weights.
\end{remark}

\section{Experimental Setup}
\label{sec:evaluations}
This section deals with the overview of experimental setup including system design, datasets used, data partitioning approach and ML models used.
\subsection{System Design}
\label{sec:evaluations:system_design}
The federated setup designed for experimentation is developed using Python 3, PyTorch, and FedEasy \cite{fedeasy} is used as a federated framework. The setup consists of one server and 100 clients. Experiments are done on a system with an AMD® Ryzen Threadripper 3960X 24-core CPU with Nvidia GeForce RTX 3060 GPU and 256GB RAM.

\subsection{Datasets and Data Partitioning}
\label{sec:evaluations:datasets}
We evaluate \textit{FedLBW} on three standard image‐classification benchmarks viz: FashionMNIST, CIFAR‐10, and CIFAR‐100, which span increasing dataset complexity and are widely adopted in FL research.  To replicate realistic federated scenarios, we distribute each dataset across $C$ clients using a Dirichlet partition \(\mathrm{Dir}(\alpha)\) on the label space.  The concentration parameter \(\alpha\) directly controls the degree of data heterogeneity: smaller values produce more skewed (non‐IID) splits.  In particular, we use $\alpha = 0.1,\;0.3,\;\text{and}\;0.6$,
to represent extreme, moderate, and mild non‐IID conditions, respectively.  At \(\alpha=0.1\), most clients receive samples from only a few classes—closely approximating single‐class heterogeneity—while still maintaining sufficient data per client for stable training.  Values of \(\alpha=0.3\) and \(\alpha=0.6\) then span a continuum from moderate to mild skew.  This design allows us to assess \textit{FedLBW’s} effectiveness across a spectrum of non‐IID settings and dataset complexities.

To create the server-side proxy dataset for \textit{FedLBW}, we followed an approach consistent with established FL methods like FedLAW \cite{li2023revisiting}. We use 100 samples per class for FashionMNIST and CIFAR-10 (1,000 total), and 10 samples per class for CIFAR-100 (1,000 total). This compact, class‑balanced validation dataset ensures fair representation across all categories without inflating size, which is crucial for fair evaluation of client models.
These proxy samples are assembled entirely on the server from publicly available or held‐out data and never include any client‐owned records.  During each communication round, the server uses this proxy solely to compute each client’s validation loss, which determines the inverse‐loss weighting coefficient.  Under no circumstances are proxy samples used for local training or global model updates, nor do they contribute gradients; they serve exclusively to evaluate client models.  By restricting the proxy to a validation role, we preserve FL’s privacy guarantees while still leveraging performance‐based aggregation.

An overview of the datasets and data distributions used along with some statistics are given in Table \ref{tab:data_stats}.

\begin{table}[!h]
\begin{center}
\caption{Statistics about datasets and data distributions used for training.}
\label{tab:data_stats}
\resizebox{0.9\linewidth}{!}{%
\begin{tabular}{ccllrrrr} 
\toprule
\multicolumn{1}{l}{\multirow{2}{*}{\textit{\textbf{Dataset}}}}           & \multicolumn{1}{l}{\multirow{2}{*}{\textit{\textbf{\#Classes}}}} & \multirow{2}{*}{\textbf{\textit{Shape}}} & \multicolumn{1}{l|}{\multirow{2}{*}{\textit{\textbf{Data Distribution}}}} & \multicolumn{4}{c}{\textit{\textbf{Training Samples per Client}}}                                                                                                                  \\ 
\cline{5-8}
\multicolumn{1}{l}{}                                                     & \multicolumn{1}{l}{}                                             &                                          & \multicolumn{1}{l|}{}                                                     & \multicolumn{1}{c}{\textit{\textbf{min}}} & \multicolumn{1}{c}{\textit{\textbf{max}}} & \multicolumn{1}{c}{\textit{\textbf{mean}}} & \multicolumn{1}{c}{\textit{\textbf{stddev}}}  \\ 
\midrule
\multirow{3}{*}{\begin{tabular}[c]{@{}c@{}}FashionMNIST\end{tabular}} & \multirow{3}{*}{10}                                              & \multirow{3}{*}{(1, $28 \times 28$)}     & Non-IID ($\alpha = 0.1$)                                                  & 18                                        & 1692                                      & 600                                        & 343.04                                        \\
                                                                         &                                                                  &                                          & Non-IID ($\alpha = 0.3$)                                                  & 142                                       & 1151                                      & 600                                        & 232.00                                        \\
                                                                         &                                                                  &                                          & Non-IID ($\alpha = 0.6$)                                                  & 211                                       & 1000                                      & 600                                        & 158.75                                        \\ 
\midrule
\multirow{3}{*}{CIFAR-10}                                                & \multirow{3}{*}{10}                                              & \multirow{3}{*}{(3, $32 \times 32$)}     & Non-IID ($\alpha = 0.1$)                                                  & 14                                        & 1409                                      & 500                                        & 258.77                                        \\
                                                                         &                                                                  &                                          & Non-IID ($\alpha = 0.3$)                                                  & 118                                       & 958                                       & 500                                        & 193.46                                        \\
                                                                         &                                                                  &                                          & Non-IID ($\alpha = 0.6$)                                                  & 180                                       & 842                                       & 500                                        & 132.48                                        \\ 
\midrule
\multirow{3}{*}{CIFAR-100}                                               & \multirow{3}{*}{100}                                             & \multirow{3}{*}{(3, $32 \times 32$)}     & Non-IID ($\alpha = 0.1$)                                                  & 268                                       & 663                                       & 500                                        & 75.32                                         \\
                                                                         &                                                                  &                                          & Non-IID (~$\alpha = 0.3$~)                                                & 356                                       & 602                                       & 500                                        & 42.58                                         \\
                                                                         &                                                                  &                                          & Non-IID (~$\alpha = 0.6$~)                                                & 394                                       & 570                                       & 500                                        & 28.11                                         \\
\bottomrule
\end{tabular}
}
\end{center}
\end{table}

\subsection{Models and Hyper-Parameters}
\label{sec:evaluations:hps}
We evaluate our approach using different model architectures (CNN for FashionMNIST, ResNet-18 for CIFAR-10, and ResNet-34 for CIFAR-100). CNN is a relatively simple model with 819,663 parameters, while ResNet-18 and ResNet-34 have 11,173,962 and 21,328,292 parameters, respectively. The rationale behind using different model architectures with varying complexities is to assess the efficiency of our proposed algorithm across diverse datasets and models, ranging from simple to complex scenarios. The CPR is set at 0.1, meaning that only 10 randomly selected clients will participate in each round. We utilized SGD with a learning rate of 0.01. The batch size is set to 32, and each round consists of 3 local epochs. All the experiments were run three times with different seeds for 300 rounds unless specified otherwise.

\subsection{Baselines}
\label{sec:evaluations:baselines}
To evaluate the effectiveness of our proposed \textit{FedLBW}, we compare its performance against several state-of-the-art FL algorithms. The baseline methods are selected to represent different approaches to handling non-IID data and client heterogeneity in FL:

\begin{itemize}
\item \textbf{FedAvg \cite{BrendanMcMahan2017}:} A foundational FL algorithm that weights client updates based solely on dataset size. While simple and widely adopted, FedAvg often struggles with non-IID data distributions and is sensitive to client dropouts.

\item \textbf{FedAvgM \cite{Hsu2019}:} An extension of FedAvg that incorporates momentum during the aggregation process to stabilize training and potentially improve convergence. The momentum term helps maintain consistent update directions across training rounds.

\item \textbf{FedProx \cite{li2020federated}:} This method addresses client heterogeneity by adding a proximal term to the local optimization objective, which helps constrain local updates to stay close to the global model. This modification aims to improve stability in non-IID settings.

\item \textbf{FedNova \cite{wang2021novel}:} FedNova tackles the issue of client step‐size heterogeneity by normalizing each client’s model update according to the number of local iterations.  This normalization prevents clients with large local workloads from disproportionately influencing the global aggregation, thereby stabilizing convergence under non‐IID data.

\item \textbf{FedLAW \cite{li2023revisiting}:} A sophisticated approach that challenges traditional normalized weight aggregation in FL by introducing learnable aggregation weights. It incorporates additional weighting strategies to handle non-IID data.

\item \textbf{FedDkw \cite{feddkw}:}  FedDkw assigns dynamic weights to clients based on the KL divergence between their local and global data distributions. Clients whose distribution is closer to global one is assigned higher weights and vice-versa.
\end{itemize}

\par
Each baseline represents a different strategy for addressing the challenges in FL, providing a comprehensive framework for evaluating \textit{FedLBW's} performance.

\section{Results and Discussion}
\label{sec:results_discussions}
In this section, the results are presented from comprehensive comparative experiments conducted to verify the effectiveness of our \textit{FedLBW} algorithm. 
In the comparative experiments, we benchmark our proposed \textit{FedLBW} approach against several state-of-the-art FL algorithms, including FedAvg, FedAvgM, FedProx, FedNova, FedLAW and FedDkw. The results, presented in the following subsections, demonstrate the superior performance of our proposed approach in terms of test accuracy and convergence speed, particularly in extreme non-IID data settings and client dropouts.

\subsection{Performance Comparison Across Datasets and Non-IID Distributions}
Table~\ref{tab:results} presents the average test accuracy achieved by the compared algorithms on three benchmark datasets, using different model architectures and varying degrees of data heterogeneity controlled by the Dirichlet parameter $\alpha$. As $\alpha$ decreases, the data distribution becomes increasingly non-IID, making the federated optimization problem more challenging. Each experiment is conducted over 300 rounds, and the reported values represent the mean and standard deviation of the final test accuracy (averaged over the last ten rounds) from three independent runs with different random seeds. The proposed \textit{FedLBW} consistently outperforms all baseline methods across datasets and non-IID configurations, with the best results highlighted in bold.

\begin{table}[!t]
\centering
\caption{Comparison of test accuracy (\%) achieved by the proposed \textit{FedLBW} with other FL algorithms on three benchmark datasets with varying model architectures, using data distribution controlled by Dirichlet parameter Dir($\alpha$).The best-performing algorithm for each dataset and non-IID setting is highlighted in bold.}
\label{tab:results}
 \resizebox{1.05\columnwidth}{!}{%
\begin{tblr}{
  row{1} = {c},
  cell{2}{1} = {r=3}{c},
  cell{2}{2} = {r=3}{c},
  cell{2}{3} = {c},
  cell{2}{4} = {c},
  cell{2}{5} = {c},
  cell{2}{6} = {c},
  cell{2}{10} = {c},
  cell{3}{3} = {c},
  cell{3}{4} = {c},
  cell{3}{5} = {c},
  cell{3}{6} = {c},
  cell{3}{10} = {c},
  cell{4}{3} = {c},
  cell{4}{4} = {c},
  cell{4}{5} = {c},
  cell{4}{6} = {c},
  cell{4}{10} = {c},
  cell{5}{1} = {r=3}{c},
  cell{5}{2} = {r=3}{c},
  cell{5}{3} = {c},
  cell{5}{4} = {c},
  cell{5}{5} = {c},
  cell{5}{6} = {c},
  cell{5}{10} = {c},
  cell{6}{3} = {c},
  cell{6}{4} = {c},
  cell{6}{5} = {c},
  cell{6}{6} = {c},
  cell{6}{10} = {c},
  cell{7}{3} = {c},
  cell{7}{4} = {c},
  cell{7}{5} = {c},
  cell{7}{6} = {c},
  cell{7}{10} = {c},
  cell{8}{1} = {r=3}{c},
  cell{8}{2} = {r=3}{c},
  cell{8}{3} = {c},
  cell{8}{4} = {c},
  cell{8}{5} = {c},
  cell{8}{6} = {c},
  cell{8}{10} = {c},
  cell{9}{3} = {c},
  cell{9}{4} = {c},
  cell{9}{5} = {c},
  cell{9}{6} = {c},
  cell{9}{10} = {c},
  cell{10}{3} = {c},
  cell{10}{4} = {c},
  cell{10}{5} = {c},
  cell{10}{6} = {c},
  cell{10}{10} = {c},
  hline{1,11} = {-}{0.08em},
  hline{2,5,8} = {-}{},
}
\textbf{Dataset} & \textbf{Model} & \textbf{$\alpha$} & \textbf{FedAvg} & \textbf{FedAvgM} & \textbf{FedProx} & \textbf{FedNova} & \textbf{FedLAW} & \textbf{FedDkw} & \textbf{FedLBW}\\
{\textit{Fashion}\\\textit{MNIST}} & \textit{CNN} & 0.1 & $81.40 \pm 1.84$ & $82.07 \pm 1.51$ & $81.19 \pm 1.61$ & $82.52 \pm 1.64$ & $82.82 \pm 1.43$ & $82.02 \pm 2.10$ & $\mathbf{83.51 \pm 1.48}$\\
 &  & 0.3 & $86.59 \pm 0.42$ & $86.55 \pm 0.45$ & $86.73 \pm 0.37$ & $86.43 \pm 0.77$ & $87.06 \pm 0.51$ & $86.69 \pm 0.73$ & $\mathbf{87.22 \pm 0.50}$\\
 &  & 0.6 & $87.08 \pm 0.75$ & $87.30 \pm 0.38$ & $87.25 \pm 0.70$ & $87.36 \pm 0.49$ & $87.71 \pm 0.47$ & $87.46 \pm 0.36$ & $\mathbf{87.72 \pm 0.33}$\\
{\textit{CIFAR}\\\textit{10}} & {\textit{ResNet}\\\textit{18}} & 0.1 & $50.97 \pm 3.55$ & $53.35 \pm 3.26$ & $48.79 \pm 3.29$ & $52.68 \pm 3.28$ & $49.19 \pm 2.71$ & $55.95 \pm 2.50$ & $\mathbf{58.63 \pm 3.02}$\\
 &  & 0.3 & $67.60 \pm 2.53$ & $70.54 \pm 1.30$ & $69.20 \pm 1.93$ & $70.09 \pm 1.34$ & $68.81 \pm 2.58$ & $71.08 \pm 1.17$ & $\mathbf{72.20 \pm 1.40}$\\
 &  & 0.6 & $75.74 \pm 1.07$ & $75.93 \pm 1.27$ & $75.47 \pm 1.55$ & $75.05 \pm 1.33$ & $74.81 \pm 1.77$ & $75.65 \pm 0.56$ & $\mathbf{77.25 \pm 0.17}$\\
{\textit{CIFAR}\\\textit{100}} & {\textit{ResNet}\\\textit{34}} & 0.1 & $37.52 \pm 0.44$ & $38.21 \pm 0.23$ & $37.60 \pm 0.38$ & $37.95 \pm 0.22$ & $37.26 \pm 0.31$ & $38.09 \pm 0.32$ & $\mathbf{39.11 \pm 0.23}$\\
 &  & 0.3 & $40.39 \pm 0.30$ & $40.68 \pm 0.50$ & $39.72 \pm 0.41$ & $40.34 \pm 0.38$ & $40.20 \pm 0.43$ & $40.38 \pm 0.28$ & $\mathbf{41.29 \pm 0.20}$\\
 &  & 0.6 & $41.68 \pm 0.15$ & $41.94 \pm 0.28$ & $41.09 \pm 0.21$ & $41.21 \pm 0.17$ & $41.57 \pm 0.09$ & $41.32 \pm 0.24$ & $\mathbf{42.01 \pm 0.20}$
\end{tblr}
}
\end{table}

\textit{FedLBW} consistently outperforms the baselines for the FashionMNIST dataset, achieving accuracies of 83.51\%, 87.22\%, and 87.72\% for $\alpha=$0.1, 0.3, and 0.6, respectively.
The improvement is most pronounced under severe non-IID conditions ($\alpha=0.1$), where \textit{FedLBW} surpasses FedAvg and FedLAW by 2.11\% and 0.69\%, respectively. A similar trend appears in the CIFAR-10 experiments with ResNet-18, where \textit{FedLBW} achieves 7.66\%, 4.60\%, and 1.51\% gains over FedAvg across increasing $\alpha$ values, and maintains a 9.44\% advantage over FedLAW at $\alpha=0.1$. On the more complex CIFAR-100 dataset with ResNet-34, \textit{FedLBW} maintains its superiority, yielding 39.11\%, 41.29\%, and 42.01\% for $\alpha=0.1$, 0.3, and 0.6. These consistent improvements across diverse datasets and architectures highlight the robustness and generalizability of \textit{FedLBW}.
\par
\paragraph{\textbf{Interpretation}}
The superior performance of \textit{FedLBW} across all datasets stems from its fundamental design principle i.e dynamically adjusting client aggregation weights based on validation loss rather than data size. This adaptive weighting allows the server to assign greater influence to clients whose models generalize well, while reducing the contribution of those with higher losses that may stem from biased or noisy local data. As a result, the aggregated global model benefits from gradients that are more consistent and representative of the overall data distribution, leading to improved generalization, smoother optimization, and higher final accuracy compared to other baselines.
\par
Building on this general advantage, \textit{FedLBW} exhibits its strongest gains under highly non-IID conditions (e.g., $\alpha=0.1$), where data heterogeneity among clients is most pronounced. In such cases, many clients have samples from only one or two classes, producing local models that overfit and yield high validation losses on the server’s balanced proxy dataset. By down-weighting these clients during aggregation, \textit{FedLBW} prevents the global model from being dominated by class-specific or biased updates, effectively stabilizing learning and preserving a balanced global representation. This explains the 7.66\% improvement over FedAvg on CIFAR-10 at $\alpha=0.1$, which diminishes to 1.51\% at $\alpha=0.6$ as the data distribution becomes more uniform and the disparity between client updates decreases. Furthermore, the continued superiority of \textit{FedLBW} over FedLAW (9.44\% at $\alpha=0.1$) underscores the robustness of its simple inverse-loss weighting strategy. Unlike approaches that learn aggregation weights through additional optimization, which may become unstable in highly heterogeneous scenarios, \textit{FedLBW} achieves stable and effective weighting through a direct, loss-driven criterion.
\par
Figures \ref{fig:fmnist_acc}-\ref{fig:resnet34_cifar100_acc} illustrate the convergence behavior of the compared algorithms by plotting the test accuracy over 300 training rounds for the three datasets under different non-IID data distributions. These figures provide valuable insights into the performance of our proposed \textit{FedLBW} approach and the baselines. Each figure comprises three sub-figures (a, b, and c) representing data distributions of $\alpha$ = 0.1, 0.3, and 0.6, respectively.
\par
\begin{figure}[!b]
\centering
\begin{subfigure}[b]{0.32\linewidth}
    \includegraphics[width=1.1\columnwidth, keepaspectratio]{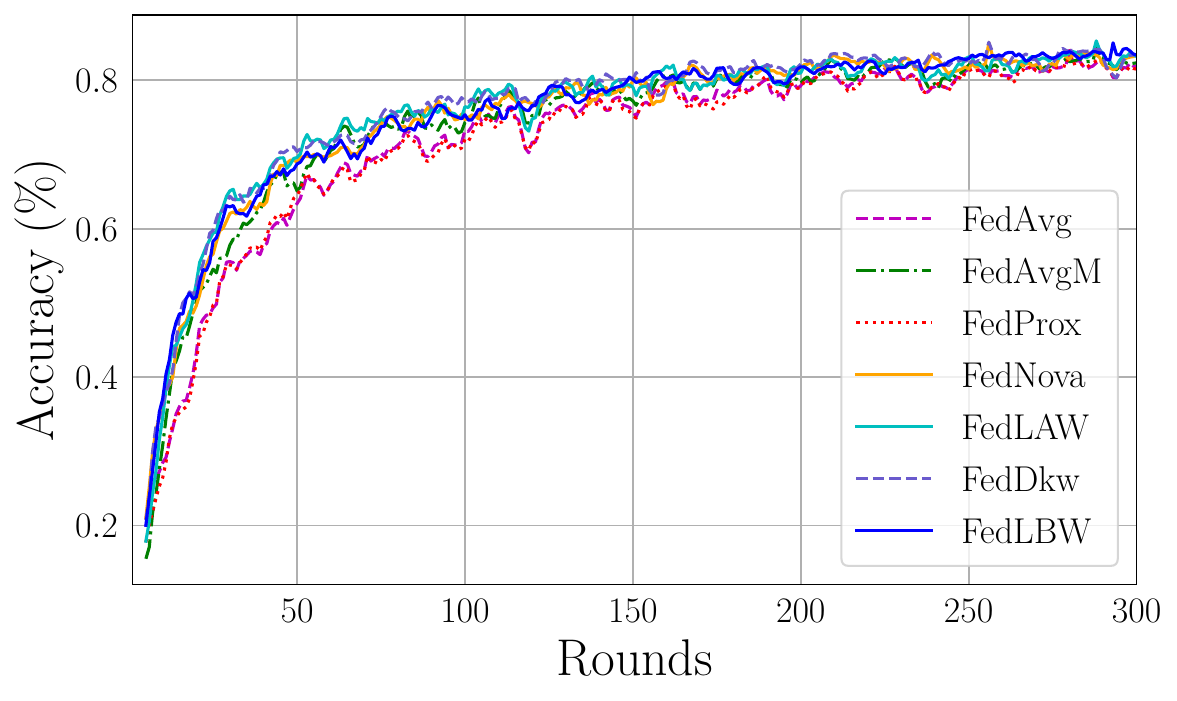}
    \caption{ $\alpha$ = 0.1}
    \label{fig:fmnist_acc_0_1}
\end{subfigure}
\hfill
\begin{subfigure}[b]{0.32\linewidth}
    \includegraphics[width=1.1\columnwidth,keepaspectratio]{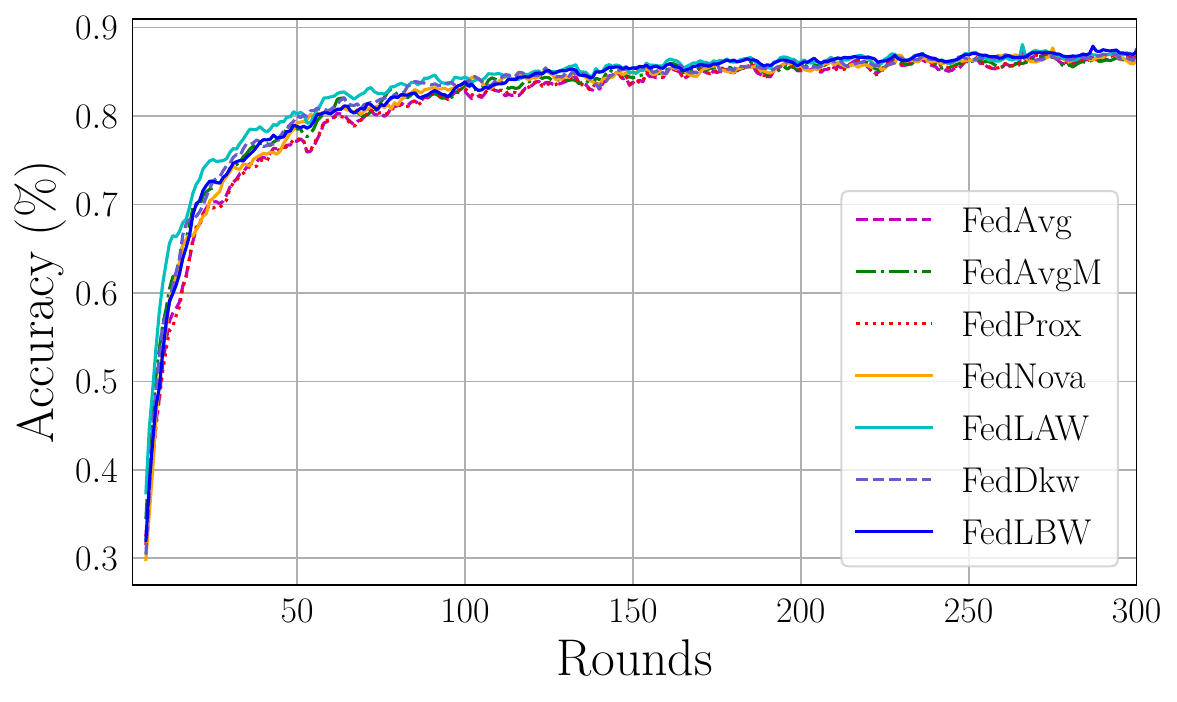}
    \caption{$\alpha$ = 0.3}
    \label{fig:fmnist_acc_0_3}
\end{subfigure}
\hfill
\begin{subfigure}[b]{0.32\linewidth}
    \includegraphics[width=1.1\columnwidth,keepaspectratio]{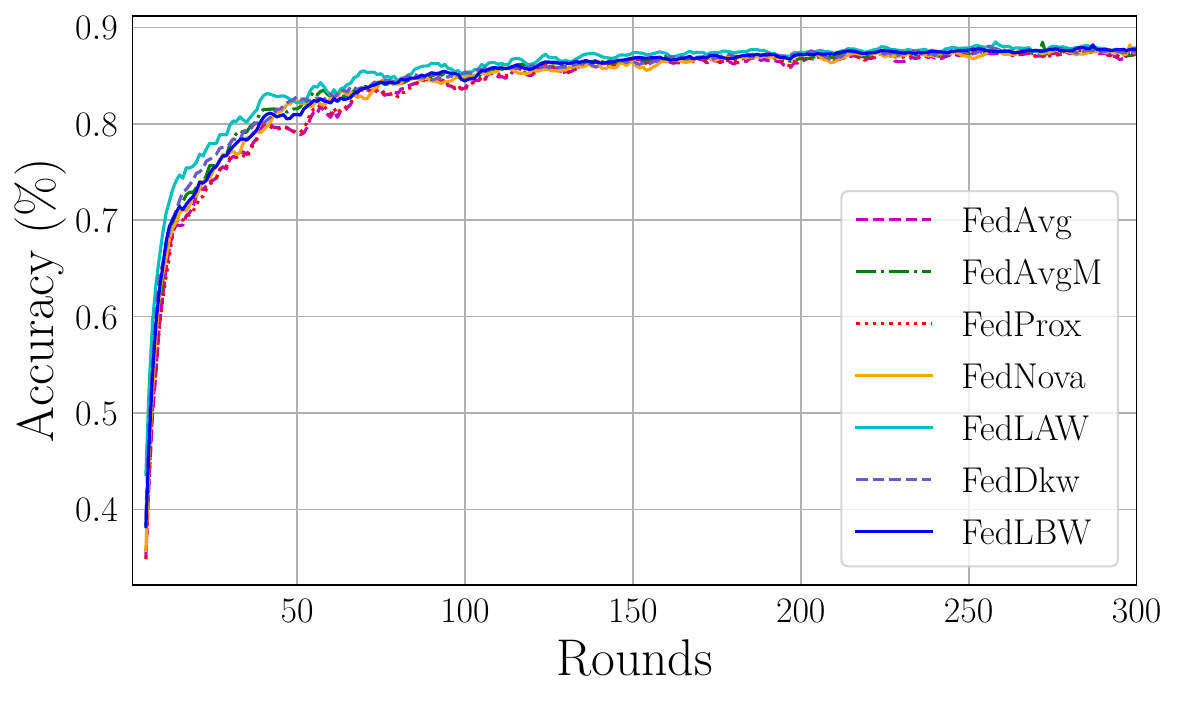}
    \caption{$\alpha$ = 0.6}
    \label{fig:fmnist_acc_0_6}
\end{subfigure}
\caption{Global accuracy over rounds for the FashionMNIST dataset using a CNN model architecture under different levels of non-IID data distribution controlled by the Dirichlet parameter ($\alpha$)}
\label{fig:fmnist_acc}
\end{figure}

\begin{figure}[!h]
\centering
\begin{subfigure}[b]{0.32\linewidth}
    \includegraphics[width=1.1\columnwidth, keepaspectratio]{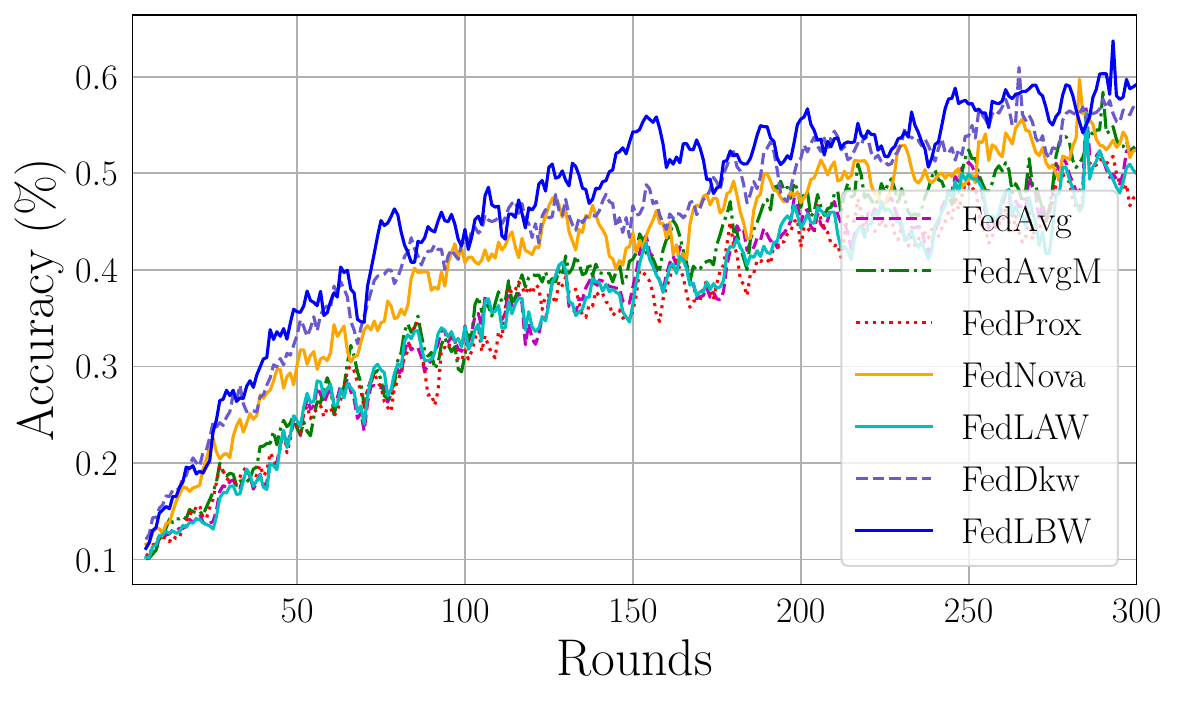}
    \caption{$\alpha$ = 0.1}
    \label{fig:resnet18_cifar10_acc_0_1}
\end{subfigure}
\hfill
\begin{subfigure}[b]{0.32\linewidth}
    \includegraphics[width=1.1\columnwidth,keepaspectratio]{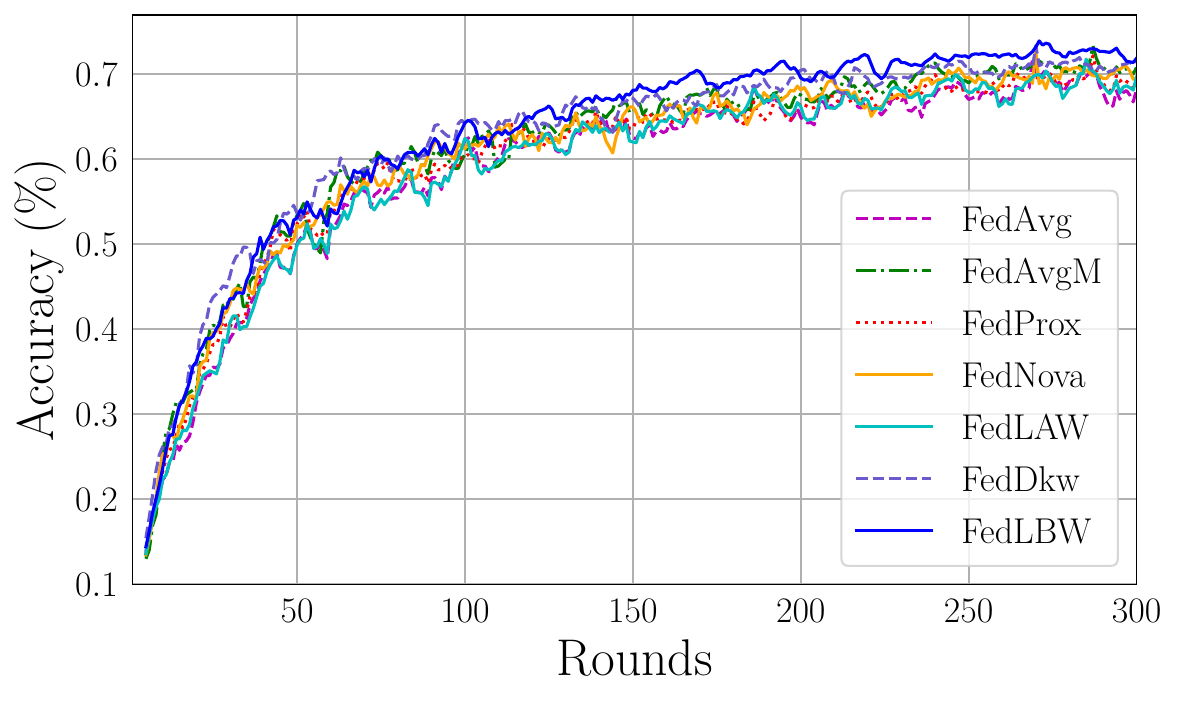}
    \caption{$\alpha$ = 0.3}
    \label{fig:resnet18_cifar10_acc_0_3}
\end{subfigure}
\hfill
\begin{subfigure}[b]{0.32\linewidth}
    \includegraphics[width=1.1\columnwidth,keepaspectratio]{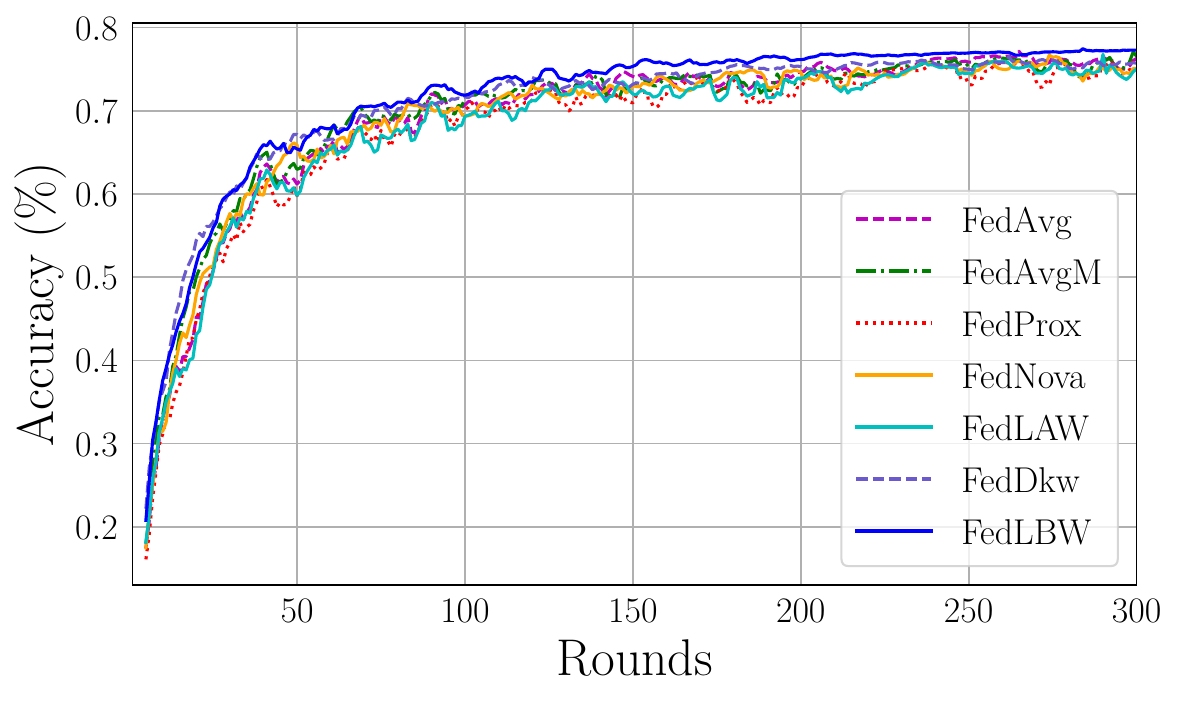}
    \caption{$\alpha$ = 0.6}
    \label{fig:resnet18_cifar10_acc_0_6}
\end{subfigure}
\caption{Global accuracy over rounds for the CIFAR-10 dataset using a ResNet-18 architecture under different levels of non-IID data distribution controlled by the Dirichlet parameter ($\alpha$)}

\label{fig:resnet18_cifar10_acc}
\end{figure}

\begin{figure}[!h]
\centering
\begin{subfigure}[b]{0.32\linewidth}
    \includegraphics[width=1.1\columnwidth, keepaspectratio]{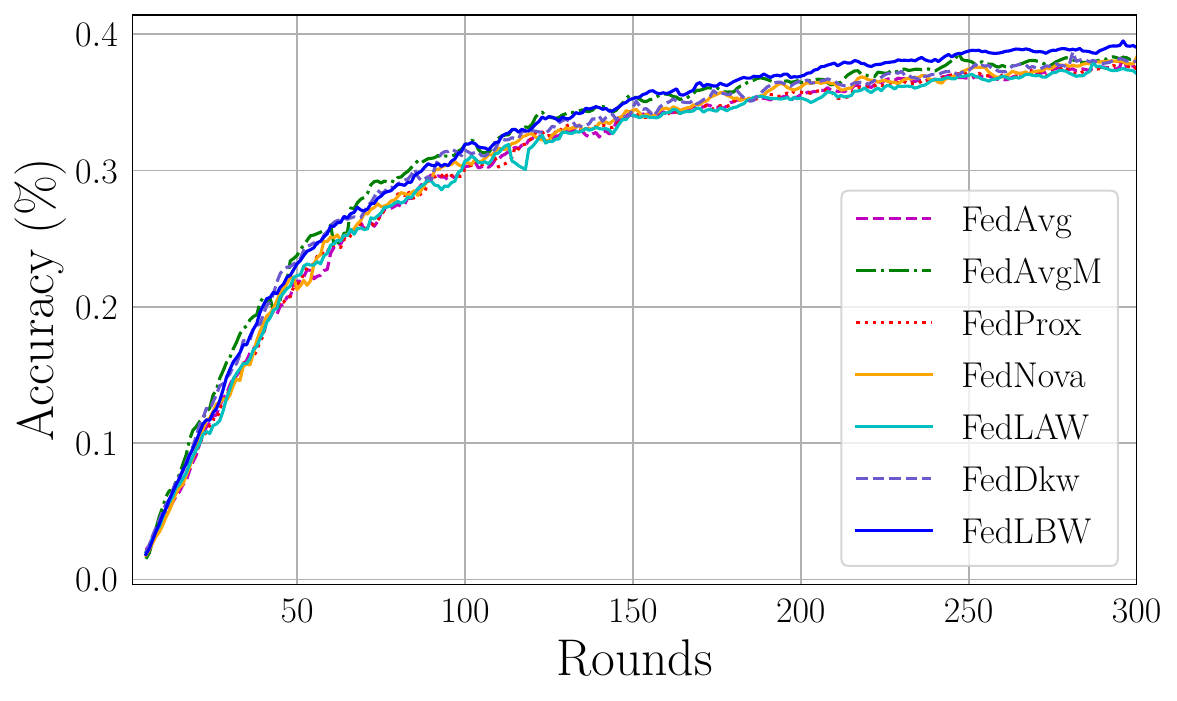}
    \caption{$\alpha$ = 0.1}
    \label{fig:resnet34_cifar100_acc_0_1}
\end{subfigure}
\hfill
\begin{subfigure}[b]{0.32\linewidth}
    \includegraphics[width=1.1\columnwidth,keepaspectratio]{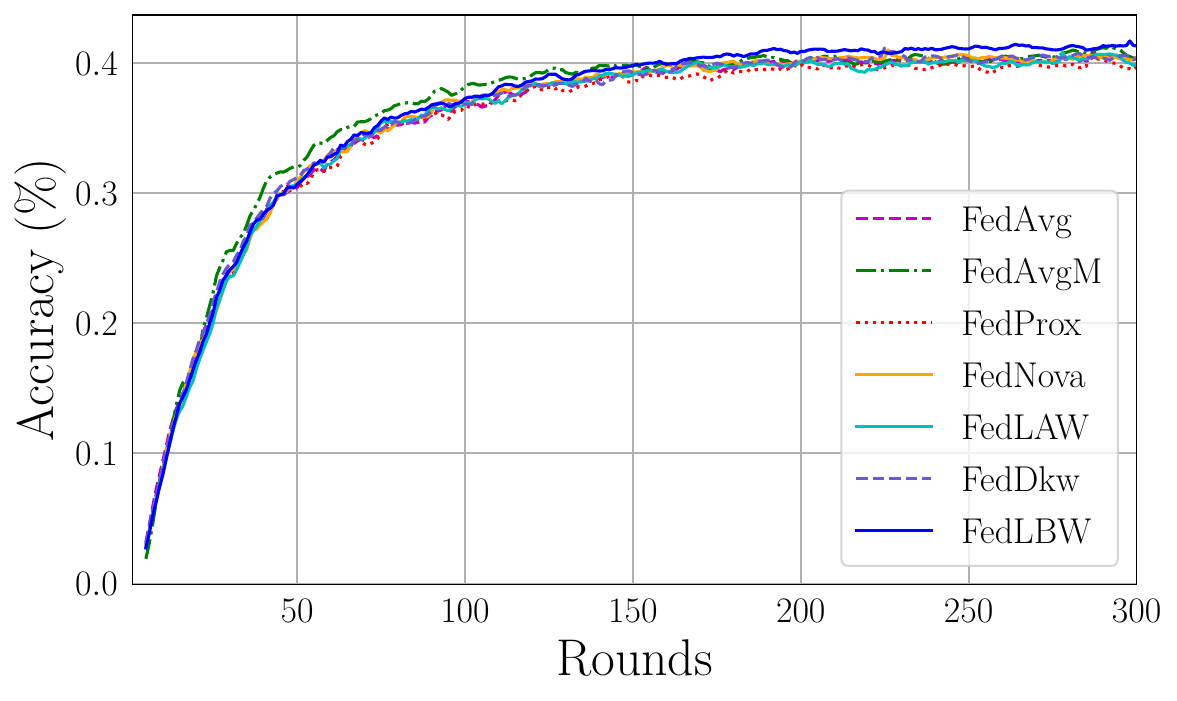}
    \caption{$\alpha$ = 0.3}
    \label{fig:resnet34_cifar100_acc_0_3}
\end{subfigure}
\hfill
\begin{subfigure}[b]{0.32\linewidth}
    \includegraphics[width=1.1\columnwidth,keepaspectratio]{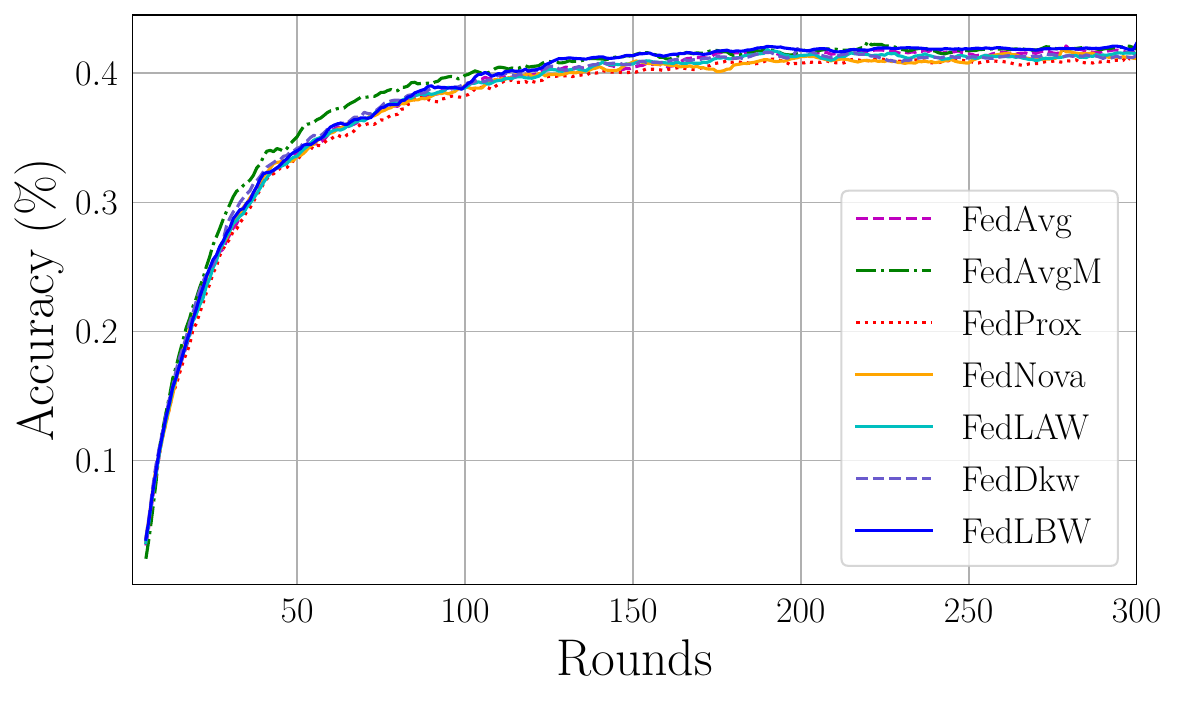}
    \caption{$\alpha$ = 0.6}
    \label{fig:resnet34_cifar100_acc_0_6}
\end{subfigure}
\caption{Global accuracy over rounds for the CIFAR-100 dataset using a ResNet-34 architecture under different levels of non-IID data distribution controlled by the Dirichlet parameter ($\alpha$)}
\label{fig:resnet34_cifar100_acc}
\end{figure}

For the FashionMNIST dataset in Figure \ref{fig:fmnist_acc}, \textit{FedLBW} converges faster and achieves higher test accuracy compared to the other algorithms. particularly in the highly non-IID setting $\alpha$ = 0.1. This trend is consistent across different non-IID levels, highlighting the robustness of our approach to varying degrees of data distribution heterogeneity.
The convergence plots for the CIFAR-10 dataset in Figure \ref{fig:resnet18_cifar10_acc} further reinforce the superiority of our proposed \textit{FedLBW} approach. \textit{FedLBW} not only achieves higher final test accuracy but also exhibits faster convergence compared to the baselines, especially in the highly non-IID setting $\alpha$ = 0.1. This behavior can be attributed to the ability of our approach to mitigate the influence of poorly trained local models and emphasize well-performing clients during the aggregation process.
For the more complex CIFAR-100 dataset in Figure \ref{fig:resnet34_cifar100_acc}, \textit{FedLBW} consistently outperforms the baselines in terms of both final test accuracy and convergence speed across all non-IID settings. The performance gap between \textit{FedLBW} and the other algorithms is particularly notable in the highly non-IID case $\alpha$ = 0.1, where the baselines struggle to converge effectively due to the high degree of data heterogeneity.

\par
The faster convergence, particularly visible in early training rounds under high non-IID conditions, stems from the improved quality of aggregated gradients. In standard FedAvg, data-size weighting can amplify gradients from clients with large but biased datasets, causing oscillatory behavior (as visible in the initial rounds of Figure~\ref{fig:resnet18_cifar10_acc}(a)). In contrast, the loss-based weighting in \textit{FedLBW} implicitly evaluates gradient quality: clients with better generalization exert stronger influence, resulting in a more coherent descent direction. The smoother convergence curves, especially under moderate non-IID settings ($\alpha=0.3$), validate this stabilization effect. Overall, these findings confirm that loss-based weighting provides a simple yet powerful mechanism for achieving both high accuracy and stable optimization in FL under diverse non-IID conditions.

\subsection{Impact of Client Dropouts}
Client dropout is a common challenge in wireless FL, where unstable connectivity and device interruptions can significantly affect convergence and global model performance. We hypothesize that \textit{FedLBW} exhibits greater resilience to such dropouts due to its loss-based aggregation approach. To validate this, we evaluate \textit{FedLBW} against FedAvg and FedLAW under dropout probabilities of 0.0, 0.1, 0.2, and 0.5 using the CIFAR-10 dataset with a ResNet-18 model. The Dirichlet parameter $\alpha$ is fixed at 0.3, training runs for 200 rounds, and other hyper-parameters follow Section~\ref{sec:evaluations:hps}. Dropouts are simulated by randomly deactivating clients each round according to the specified probability.
\par
Table~\ref{tab:client_dropout} summarizes the results. \textit{FedLBW} consistently achieves the highest accuracy across all dropout levels, starting at 70.65\% when all clients participate and maintaining a roughly 3\% advantage over FedAvg and FedLAW. As dropout probability increases to 0.1 and 0.2, all methods show moderate declines, yet \textit{FedLBW} retains accuracies of 69.47\% and 68.44\%, respectively. Under severe dropout ($p=0.5$), FedAvg drops sharply to 39.81\%, whereas \textit{FedLBW} and FedLAW remain stable at 64.40\% and 63.72\%. Thus, \textit{FedLBW} preserves over 90\% of its baseline performance even when half the clients disconnect.
\par

\begin{table}[!b]
\centering
\caption{Test accuracy(\%) based on different client dropout probabilities on CIFAR-10 using ResNet-18 with Dir($\alpha$) = 0.3, E = 3 and 200 rounds.}
\label{tab:client_dropout}
\footnotesize
\begin{tblr}{
  cells = {r},
  row{1} = {c},
  cell{2}{1} = {c},
  cell{3}{1} = {c},
  cell{4}{1} = {c},
  cell{5}{1} = {c},
  cell{6}{1} = {c},
  hline{1,6} = {-}{0.08em},
  hline{2} = {-}{},
}
{\textbf{Dropout}\textbf{ Probability}} & \textbf{FedAvg} & \textbf{FedLAW} & \textbf{FedLBW}\\
\textit{0.0}  & 67.84 ± 1.62     & 67.87 ± 2.73      & \textbf{70.65 ± 1.74}  \\
\textit{0.1}  & 67.64 ± 1.49     & 67.01 ± 1.91       & \textbf{69.47 ± 1.90}   \\
\textit{0.2}  & 66.61 ± 1.35      & 66.27 ± 1.61       & \textbf{68.44 ± 2.15}   \\
\textit{0.5}  & 39.81 ± 4.25      & 63.72 ± 3.30        & \textbf{64.40 ± 4.26}            
\end{tblr}

\end{table}

\par
\paragraph{\textbf{Interpretation}}
\textit{FedLBW}'s robustness to dropouts arises from its adaptive, loss-based weighting, which prioritizes reliable client updates rather than number of data samples. Clients achieving lower validation losses contribute proportionally more to the global update, ensuring that when dropouts occur, aggregation still reflects well-generalized models. This selective emphasis minimizes bias from missing or noisy participants and stabilizes training.
\par
Dropouts typically distort the global gradient direction, as the remaining clients no longer represent the full data distribution. \textit{FedLBW} mitigates this by effectively performing a form of gradient reliability filtering—clients with lower validation losses tend to produce updates aligned with the global descent path. Consequently, even with reduced participation, the aggregated gradient remains consistent, explaining the small 6.25\% decline in performance from $p=0.0$ to $p=0.5$, compared to FedAvg’s 28.03\% drop.
\par
Moreover, compared to FedLAW, which learns dynamic aggregation weights, \textit{FedLBW}'s simpler inverse-loss weighting proves more robust in dynamic conditions. Learned weighting schemes may overfit transient client behaviors when participation fluctuates, whereas \textit{FedLBW} maintains stability through a static yet performance-aware weighting rule. Overall, these findings validate our hypothesis about \textit{FedLBW}’s resilience to client dropouts, confirming that \textit{FedLBW}’s weighting mechanism implicitly acts as a reliability estimator, emphasizing generalizable and stable updates during aggregation, making it particularly suitable for cross-device and wireless FL scenarios.

\subsection{Impact of CPR and Total Clients}
Table \ref{tab:cpr_ablation} provides a comprehensive evaluation of the impact of CPR and the total number of clients on the performance of various FL algorithms, such as FedAvg, FedLAW, and \textit{FedLBW}. The experiments are conducted on the CIFAR-10 dataset using ResNet-18, with a Dirichlet distribution parameter $\alpha = 0.1$, 3 local epochs per round, and 150 total communication rounds, other hyper-parameters are same as that discussed in Sections \ref{sec:evaluations:system_design} and \ref{sec:evaluations:hps}. 

\begin{table}[!b]
\centering
\caption{Test accuracy(\%) for different total number of clients and client participation ratios on CIFAR-10 using ResNet-18 with Dir($\alpha$) = 0.1, E = 3 and total rounds = 150}
\label{tab:cpr_ablation}
\footnotesize
\begin{tblr}{
  cells = {r},
  row{1} = {c},
  cell{2}{1} = {c},
  cell{3}{1} = {c},
  cell{4}{1} = {c},
  cell{5}{1} = {c},
  cell{6}{1} = {c},
  hline{1,7} = {-}{0.08em},
  hline{2} = {-}{},
}
{\textbf{Total Clients}\textbf{ (CPR)}} & \textbf{FedAvg} & \textbf{FedLAW} & \textbf{FedLBW}       \\
\textit{100 (0.1)}                        & 43.80 ± 3.29    & 43.32 ± 3.38    & \textbf{49.89 ± 2.65} \\
\textit{20 (0.5)}                         & 47.91 ± 1.82    & 48.99 ± 3.04    & \textbf{56.51 ± 0.87} \\
\textit{20 (1.0)}                         & 55.99 ± 0.94    & 57.07 ± 1.08    & \textbf{58.80 ± 0.75} \\
\textit{50 (0.5)}                         & 56.31 ± 2.65    & 57.16 ± 2.23    & \textbf{64.87 ± 0.53} \\
\textit{50 (1.0)}                         & 59.88 ± 1.32    & 62.14 ± 0.43    & \textbf{65.49 ± 0.23} 
\end{tblr}

\end{table}

The results clearly demonstrate the superior performance of \textit{FedLBW} across all configurations, particularly in scenarios with low client participation. In the most challenging scenario, involving 100 total clients and a 0.1 CPR, \textit{FedLBW} achieves an accuracy of 49.89\%, significantly outperforming FedAvg (43.80\%) and FedLAW (43.32\%). This highlights the algorithm's robustness in handling situations where only a small subset of clients participate, which is a common challenge in real-world FL applications.

The performance of FedLAW is consistent with findings from its original study, showing improvements at higher CPR's. For example, when tested with 50 clients and full participation (CPR = 1.0), FedLAW achieves an accuracy of 62.14\%, outperforming FedAvg (59.88\%) by 2.26\%. However, as the CPR decreases to 0.5, the performance gap narrows to only 0.85\% (57.16\% versus 56.31\%), and at low CPR scenarios (e.g., 100 clients with 0.1 CPR), FedLAW's performance closely matches that of FedAvg.

In contrast, \textit{FedLBW} consistently maintains its performance advantage across all CPR's and client configurations, demonstrating marked improvements even at low participation rates (CPR = 0.1). This consistent superiority underscores \textit{FedLBW's} robustness to varying client participation and data heterogeneity—key challenges in FL—and highlights its potential to significantly enhance FL systems in practical applications where full client participation is often not feasible.

\paragraph{\textbf{Interpretation}}
The superior performance of \textit{FedLBW} under low CPR settings stems from its adaptive weighting mechanism, which counteracts the bias introduced by partial participation. When only a fraction of clients contribute each round, traditional data-size–weighted schemes such as FedAvg can overweight small, skewed client subsets. In contrast, \textit{FedLBW} assigns higher aggregation weights to clients with lower validation losses, ensuring that the global update direction reflects well-generalized models rather than overfitted or underrepresented data segments.

This behavior explains the large gains observed when participation is sparse: by prioritizing reliable client updates, \textit{FedLBW} mitigates the adverse effects of client imbalance and data heterogeneity, leading to more stable optimization. As participation grows and data diversity across rounds improves, all methods converge toward similar behavior, narrowing the performance gap—but \textit{FedLBW} consistently maintains a modest advantage due to its gradient-quality–aware weighting. These findings highlight \textit{FedLBW}'s suitability for practical FL deployments where full participation cannot be guaranteed, offering improved robustness.

\subsection{Robustness to Different Proxy Data Distributions}
This subsection investigates the impact of the server’s proxy dataset on \textit{FedLBW}. Specifically, \textit{FedLBW} computes inverse-loss weighting based on the validation loss measured on the server's proxy data, indicating that the proxy dataset affects the performance of \textit{FedLBW}. Therefore, this experiment is designed to validate \textit{FedLBW}’s robustness against various types of distribution shifts in the proxy data.
\par
To this end, we used the CIFAR-10 dataset and ten clients per round out of hundred, each holding data sampled with Dirichlet $\alpha=0.3$. For the server-side proxy, we used six types of configurations as follows:
\begin{itemize}[label={}]
\setlength{\itemsep}{1pt}
  \setlength{\parskip}{1pt}
    \item (i) a class-balanced CIFAR-10 subset used as the default setting
    \item (ii--iv) three label-skewed CIFAR-10 proxies sampled using Dirichlet concentration parameters $\alpha \in \{1.0, 0.6, 0.1\}$
    \item (v) a covariate-shift proxy obtained by applying a uniform brightness perturbation (factor $0.5$) to CIFAR-10 images; and
    \item (vi) a domain-shift proxy composed of SVHN images.
\end{itemize}

For (ii–iv), we adjusted the Dirichlet parameter to create varying levels of class imbalance in the proxy data, simulating distribution shifts between the server and clients. The covariate shift in (v) introduces a low-level visual change without altering semantic content, while the domain shift in (vi) reflects a more significant mismatch in both visual characteristics and label semantics.
\par
To quantify the deviation of the proxy data distribution from the clients’ data distribution, we compute two metrics.  
First, Spearman’s rank correlation coefficient ($\rho$) \cite{spearman1987proof} is calculated to assess the consistency in client ranking between the default proxy and each distribution-shifted proxy. Specifically, for each proxy configuration, we obtain the validation losses of all clients when evaluated on that proxy, and compare the ordering of these losses to the ordering produced by the default CIFAR-10 proxy. This comparison captures whether a shifted proxy preserves the same relative ranking of clients as the default proxy. A higher $\rho$ thus indicates that the shifted proxy maintains similar discriminative characteristics to the default proxy, reflecting alignment in underlying data distribution properties rather than just performance magnitude.
Second, Fréchet Inception Distance (FID) \cite{fid} is computed between the feature representations of each proxy dataset and those aggregated from all client data. FID quantifies the statistical distance between the proxy and the global client distribution in deep feature space, where lower values indicate higher similarity and better representativeness.

\begin{table}[!b]
\centering
\caption{ Performance of FedLBW under varying degrees of client–proxy distribution shifts.}
\label{tab:dist_shift}
\footnotesize
\begin{tblr}{
  column{5} = {r},
  cell{1}{5} = {c},
  cell{2}{3} = {r},
  cell{2}{4} = {r},
  cell{3}{3} = {r},
  cell{3}{4} = {r},
  cell{4}{3} = {r},
  cell{4}{4} = {r},
  cell{5}{3} = {r},
  cell{5}{4} = {r},
  cell{6}{3} = {r},
  cell{6}{4} = {r},
  cell{7}{3} = {r},
  cell{7}{4} = {r},
  hline{1-2,8} = {-}{},
}
\textbf{\textbf{Proxy Dataset}} & \textbf{Shift~Type} & \textbf{$\rho$} & \textbf{FID} & \textbf{Acc (\%)}\\
\textit{CIFAR-10} & None (Default) & 0.9999 & 51.07 & 72.20\\
\textit{CIFAR-10~(Dir $\alpha$ = 1.0)} & Label Skew & 0.9713 & 52.65 & 72.04\\
\textit{CIFAR-10~(Dir $\alpha$ = 0.6)} & Label Skew & 0.9794 & 53.77 & 71.40\\
\textit{CIFAR-10~(Dir $\alpha$ = 0.1)} & Label Skew & 0.9761 & 57.13 & 72.28\\
\textit{CIFAR-10~(Brighten)} & Covariate Shift & 0.8149 & 55.33 & 71.55\\
\textit{SVHN} & Domain Shift & 0.4219 & 138.37 & 70.03
\end{tblr}

\end{table}

Table~\ref{tab:dist_shift} reports the Spearman rank correlation coefficient $\rho$, FID, and final global test accuracy for each proxy configuration.  
In the default setting, where the proxy is a class-balanced subset of CIFAR-10, alignment is strongest ($\rho=0.9999$, FID=51.07) and the model achieves its highest accuracy of 72.20\%.  
Introducing in‐domain label skew with $\alpha=1.0$ slightly increases FID to 52.65 while $\rho$ remains high (0.9713) and accuracy is almost unchanged (72.04\%).  
With $\alpha=0.6$, FID increases further to 53.77, accompanied by a modest drop in accuracy to 71.40\%, despite $\rho$ rising to 0.9794.  
At extreme skew ($\alpha=0.1$), FID reaches 57.13, yet $\rho=0.9761$ still indicates reliable client ranking, and accuracy remains high at 72.28\%.  

Under covariate shift (brightness perturbation), FID=55.33 and $\rho=0.8149$ indicate moderate misalignment, with accuracy at 71.55\%.  
Finally, in the domain shift case using SVHN, the largest distribution gap is observed (FID=138.37) and $\rho$ drops to 0.4219, yet accuracy remains relatively strong at 70.03\%.  
Overall, higher $\rho$ values and lower FID scores generally correspond to better accuracy, with performance gradually decreasing as the severity of the distribution shift increases.  

It is worth noting that $\rho$ remains high ($>0.97$) across all in-domain label skew configurations.  
This stability arises because Spearman’s $\rho$ measures rank-order consistency rather than absolute loss differences, and label skew within the same domain (CIFAR-10) preserves the relative difficulty ordering of clients.  
Although FID increases slightly under label skew, the feature distribution overlap remains large enough for proxy losses to reliably reflect client performance rankings.  
In contrast, covariate shift (brightness perturbation) and especially domain shift (SVHN) substantially alter the feature statistics, leading to larger drops in $\rho$ (0.8149 and 0.4219, respectively) and correspondingly weaker—but still positive—performance. 

These results show that \textit{FedLBW} retains high performance across a wide range of proxy–client distribution shifts. While closer alignment between client and proxy datasets generally improves the reliability of loss-based aggregation, \textit{FedLBW} demonstrates robustness even under substantial divergence, including severe domain shifts.
\par
\paragraph{\textbf{\textit{Summary}}}
Overall, our extensive experimental results demonstrate the effectiveness of \textit{FedLBW} in improving FL system performance. By assigning higher weights to well-performing local models during aggregation, our approach effectively handles non-IID data distributions, mitigates the influence of outliers, manages client dropouts, and incentivizes better local training. These advantages translate into faster convergence, higher test accuracy, and improved robustness compared to traditional weighting approaches and other state-of-the-art FL algorithms. Notably, \textit{FedLBW} performs better as data non-IIDness increases, making it suitable for extreme non-IID cases and scenarios with high client dropout rates.

The superior performance of \textit{FedLBW} can be attributed to its direct addressing of traditional sample-based weighting limitations, such as bias towards larger datasets, ineffectiveness for non-IID data, and sensitivity to outliers. By emphasizing well-trained local models, \textit{FedLBW} effectively mitigates these issues. Theoretically, this weighting strategy aligns with ensemble learning principles, where individual model contributions are weighted by performance, leading to improved accuracy and robustness. The practical implications are significant: \textit{FedLBW} offers a versatile, broadly applicable approach that can be integrated into different FL frameworks, making it a valuable tool for enhancing model accuracy and convergence across diverse real-world domains.

\section{Conclusions and Future Scope}
\label{sec:conclusions}
This study presents \textit{FedLBW}, a novel loss-based weighting strategy for model aggregation in FL, which effectively addresses the limitations of traditional weighting approaches. 
Our comprehensive empirical evaluations demonstrate \textit{FedLBW's} superior performance across diverse datasets, model architectures, and challenging non-IID data distributions. 
The proposed approach not only enhances model accuracy and convergence but also exhibits remarkable resilience to client dropouts—maintaining stability even under extreme dropout rates, where conventional methods like FedAvg experience catastrophic degradation.
This resilience is particularly significant for wireless network deployments, where client disconnections and non-IID data distributions are inevitable challenges.
Extensive ablation studies validate \textit{FedLBW's} scalability and efficiency in fostering collaborative learning environments, establishing it as a practical and reliable solution for real-world FL applications in wireless networks.

While the proposed \textit{FedLBW} showcases promising results, further research is required to explore alternative weighting strategies and their potential integration with other techniques. Investigation of alternative weighting functions or combinations of weighting factors could lead to even further improvements in the performance and robustness of FL systems, particularly in challenging scenarios with highly non-IID data or specific application domains especially in wireless network environments.

\section*{CRediT authorship contribution statement}
\textbf{Majid Kundroo:} Writing – original draft, Software, Methodology, Conceptualization. 
\textbf{Tinku Singh:} Writing – original draft, Data curation, Methodology.
\textbf{Taehong Kim:} Writing – review \& editing, Supervision, Conceptualization. 

\section*{Declaration of competing interest}
The authors declare that they have no known competing financial interests or personal relationships that could have appeared to influence the work reported in this paper.

\section*{Data availability}
The datasets used in this study are all open source and their source has been duly cited in the paper.

\section*{Funding}
This work was supported by the Institute of Information \& communications Technology Planning \& Evaluation (IITP) grant funded by the Korean government (MSIT) (RS-2025-02217801, Development of core technologies for integrated management of large-scale on-device AI objects and networks).

\bibliographystyle{elsarticle-num} 
\bibliography{refs}

\end{document}